\documentclass[a4paper,12pt]{article} 

\usepackage{times}
\usepackage{soul}
\usepackage{url}
\usepackage[hidelinks]{hyperref}
\usepackage[utf8]{inputenc}
\usepackage[small]{caption}
\usepackage{graphicx}
\usepackage{amsthm}
\usepackage{booktabs}
\usepackage{algorithm}
\usepackage{algorithmic}
\usepackage{subfig}
\usepackage{microtype}
\usepackage{xspace}
\usepackage{amsmath}

\usepackage[fixamsmath,disallowspaces]{mathtools}

\usepackage{microtype}

\usepackage{comment}
\usepackage{macros}
\usepackage{multicol}
\usepackage{multirow}

\newtheorem{theorem}{Theorem}

\newtheorem{corollary}[theorem]{Corollary}
\newtheorem{example}[theorem]{Example}
\newtheorem{lemma}[theorem]{Lemma}
\newtheorem{claim}[theorem]{Claim}
\newtheorem{definition}[theorem]{Definition}

\usepackage{enumerate}
\usepackage{xcolor}

\usepackage{tikz}
\usetikzlibrary{arrows}
\usepackage[disable]{todonotes}

\usepackage{array}
\newcolumntype{H}{>{\setbox0=\hbox\bgroup}c<{\egroup}@{}}

\usepackage{thmtools}
\usepackage{thm-restate}
\usepackage{cleveref}

\usepackage{apptools}
\usepackage{xpatch}
\makeatletter
\xpatchcmd{\thmt@restatable}
{\csname #2\@xa\endcsname\ifx\@nx#1\@nx\else[{#1}]\fi}
{\IfAppendix{\csname #2\@xa\endcsname}{\csname #2\@xa\endcsname\ifx\@nx#1\@nx\else[{#1}]\fi}}
{}{} 
\makeatother
\usepackage{authblk}

\usepackage{xcolor}
\colorlet{vertexTopColor}{white}

\colorlet{vertexBottomColor}{black!10}

\allowdisplaybreaks

\newcommand{\postrebuttal}[1]{\textcolor{black}{#1}}

\usepackage{macros}
\usepackage{multicol}
\usepackage{multirow}
\usepackage{comment}

\usepackage{times}
\usepackage{soul}
\usepackage{url}
\usepackage[utf8]{inputenc}
\usepackage{caption}
\usepackage{graphicx}
\usepackage{amsmath} 
\usepackage{amsthm}
\usepackage{thmtools}
\usepackage{booktabs}
\usepackage{algorithm}
\usepackage{algorithmic}
\usepackage{comment}
\usepackage{thm-restate}

\newcommand{\lefttriangle}{\hfill$\triangleleft$}

\newcommand{\claimqedsymbol}{\hfill$\blacksquare$}

\newenvironment{claimproof}
{%
	\begin{proof}[Proof of Claim]%
	}
	{%
		\hfill\claimqedsymbol\par
	\end{proof}%
}

\AtEndEnvironment{example}{\lefttriangle}

\usepackage{enumerate,enumitem}

\begin{document} 

\title{Can You Tell the Difference? Contrastive Explanations for ABox Entailments}

\author[1]{Patrick Koopmann\thanks{\texttt{p.k.koopmann@vu.nl}}}
\author[2]{Yasir Mahmood\thanks{\texttt{yasir.mahmood@uni-paderborn.de}}}
\author[2]{Axel-Cyrille Ngonga Ngomo\thanks{\texttt{axel.ngonga@upb.de}}}
\author[2]{Balram Tiwari\thanks{\texttt{balram@mail.uni-paderborn.de}}}

\affil[1]{Knowledge in Artificial Intelligence, Vrije Universiteit Amsterdam, The Netherlands}
\affil[2]{Data Science Group, Heinz Nixdorf Institute, Paderborn University, Germany}

\date{\vspace{-5ex}}  

\maketitle

	\begin{abstract}
	We introduce the notion of contrastive ABox explanations
	to answer questions of the type
	``Why is $a$ an instance of $C$, but $b$ is not?''.
	%
	While there are various approaches for explaining
	positive entailments (why is $C(a)$ entailed by the knowledge base)
	as well as missing entailments (why is $C(b)$ not entailed)
	in isolation, contrastive explanations consider both at the same time,
	which allows them to focus on the relevant commonalities and differences
	between $a$ and $b$.
	%
	We develop an appropriate notion of contrastive explanations
	for the special case of ABox reasoning with description logic ontologies, and
	analyze the computational complexity for different variants 
	under different optimality criteria, considering
	lightweight as well as more expressive description logics.
	We implemented a first method for computing
	one variant of contrastive explanations,
	and evaluated it on generated problems for realistic knowledge bases.
\end{abstract}

\section{Introduction}

A key advantage of knowledge representation systems is that they enable transparent and explainable decision-making.
For example, with an ontology formalized in a description logic (DL)~\cite{baader2003description,hitzler2009foundations} we can infer implicit information from
data (then called an \emph{ABox}) through logical reasoning, and all inferences are based on explicit statements in the ontology and data.
However, due to the expressive power of DLs and the complexity of realistic ontologies, inferences obtained through
reasoning may not always be immediately understandable.
Consequently, in recent years, significant attention has been devoted to explaining \emph{why} or \emph{why not} something is entailed by a DL knowledge base (KB).
The \emph{why} question is typically answered through \emph{justifications}~\cite{
schlobach2003non,
horridge2011justification}.
For a KB $\calK$ (consisting of ontology and ABox statements) and an entailed axiom $\alpha$, a justification is a subset minimal
$\calJ\subseteq\calK$ such that $\calJ\models \alpha$.
Other techniques for explaining \emph{why} questions include \emph{proofs}~\cite{alrabbaa2022explaining} and \emph{Craig interpolants}~\cite{DBLP:conf/jelia/Schlobach04}.
To answer a \emph{why not} question, we can use \emph{abductive reasoning}
to determine what is missing in $\calK$ to derive $\alpha$~\cite{elsenbroich2006case,peirce1878deduction}.
Research in this area for DLs encompasses \emph{ABox abduction}~\cite{Del-PintoS19,Koopmann21a},
\emph{TBox abduction}~\cite{wei2014abduction,DuWM17,HaifaniKTW22},
\emph{KB abduction}~\cite{elsenbroich2006case,DBLP:conf/kr/KoopmannDTS20} and \emph{concept abduction}~\cite{Bienvenu08},
depending on the type of entailment to be explained.

\newcommand{\Qualified}{\concept{Qualified}}
\newcommand{\publishedAt}{\concept{publishedAt}}
\newcommand{\Journal}{\concept{Journal}}
\newcommand{\Interviewed}{\concept{Interviewed}}
\newcommand{\leads}{\concept{leads}}
\newcommand{\hasFunding}{\concept{hasFunding}}
\newcommand{\Group}{\concept{Group}}
\newcommand{\PostDoc}{\concept{PostDoc}}
\newcommand{\alice}{\entity{alice}}
\newcommand{\bob}{\entity{bob}}
\newcommand{\aaai}{\entity{aaai}}
\newcommand{\aij}{\entity{aij}}
\newcommand{\corp}{\entity{nsf}}
\newcommand{\cs}{\entity{kr}}

\patrick{By using macros also for the names used in the running example, we make sure that we do not accidentally
use incoherent names as we did in the KR submission. This is generally a good practice.}

If we query a KB for a set of objects, we may wonder why some object occurs in the answer but another does not. In this context, addressing the \emph{why} and \emph{why not} questions jointly can provide more clarity than considering them in isolation.
To illustrate this, consider a simplified KB for a hiring process that determines which candidates are considered for a job interview.
The KB uses a TBox with axioms
\begin{align*}
		(a)\ &\Qualified\sqcap \exists \publishedAt.\Journal \sqsubseteq \Interviewed,\\
		(b)\ & \exists \leads.\Group \sqcup \exists\hasFunding.\top\sqsubseteq \Qualified, \\
		(c)\ & \PostDoc \sqcap \exists \leads.\Group \subsum \bot \quad
\end{align*}
stating that (a) someone who is qualified and has published at a journal gets interviewed,
(b) someone who leads a group or has funding is qualified and (c) postdocs cannot lead groups.
\patrick{Order has to be flipped: I would first put the axioms, then explain them, otherwise the }
Further, we have an ABox with assertions
\begin{align*}
   & (1)\ \publishedAt(\alice, \aij),\qquad
		(2)\ \publishedAt(\bob, \aaai),\\
		&
		(3)\ \Journal(\aij),\quad 
		(4)\ \leads(\alice, \cs),\quad
		(5)\ \Group(\cs),\\
		&
		(6)\ \hasFunding(\alice, \corp),\qquad
		(7)\ \PostDoc(\bob)
\end{align*}
stating that (1)~Alice published at AIJ, (2)~Bob published at AAAI, (3)~AIJ is a journal,
(4--5) Alice leads the group KR, and
(6) receives funding from the NSF, and
(7) Bob is a Postdoc.
This knowledge base entails $\Interviewed(\alice)$, but
not $\Interviewed(\bob)$.
We can explain why Alice was interviewed with an ABox justification, e.g. $\{(1), (3), (4), (5)\}$
(\enquote{\emph{She published at the journal AIJ and leads the KR group}}).
To explain why Bob is not interviewed, we may use ABox abduction and obtain an
answer with a fresh individual $e$:
$$\{\quad \Journal(\aaai),\qquad \hasFunding(\bob,e) \quad\}$$
(\enquote{\emph{If AAAI was a journal, and Bob received
funding, he would have been interviewed.}})
For the question \enquote{\emph{Why was Alice interviewed, but not Bob?}}, those explanations are not ideal, since they consider different reasons for being qualified (funding vs. leading a group).
A better \emph{contrastive explanation} would be:
\enquote{\emph{Alice's publication is at a journal and Bob's is not, and only Alice receives funding.}}

Formally, a contrastive explanation problem consists of a concept ($\Interviewed$), a \emph{fact} individual that is an instance of the concept ($\alice$),
and a \emph{foil} individual that is not an instance ($\bob$).
Such contrastive ABox explanation problems are also motivated in the context of \emph{concept learning}~\cite{DBLP:journals/ml/LehmannH10,DBLP:conf/ijcai/FunkJLPW19,DBLP:conf/www/HeindorfBDWGDN22}, where the aim is to learn a concept from positive and negative examples. 
Contrastive explanations allow to explain the learned concepts in the light of a positive and a negative example.
\postrebuttal{Moreover, in medical domain, such explanations allow patient-history comparisons to see why certain treatments are possible for one patient but not for the other.} 

The notion of contrastive explanations appears first in the work of Lipton~(\cite{lipton1990contrastive}).
The main theme of Lipton's work is to express an inquirer's \emph{preference} or reflect their demand regarding the \emph{context} in which an explanation is requested (e.g., explain why Bob was not interviewed \emph{in the context of Alice}, who was interviewed).
Contrastive explanations have since been considered for 
answer set programming~\cite{eiter2023contrastive} with aim of explaining why some atoms are in an answer set instead of others.  
The idea has also been used to explain classification results of machine learning
models~\cite{dhurandhar2018explanations,IgnatievNA020,stepin2021survey,Miller_2021,xai}.
A related concept are \emph{counter-factual explanations} used in machine learning~\cite{verma2020counterfactual,dandl2020multi}.
What these approaches have in common is that they look at similarities and differences at the same time, and use syntactic \emph{patterns} to highlight the differences. 
We use a similar idea in the context of ABox reasoning by using \emph{ABox patterns}
which are instantiated differently for the fact and the foil.

\paragraph{Contributions}
Our notion of contrastive ABox explanations (CEs) is quite general, and
can also allow for contradictions with the KB.
We distinguish between a syntactic and a semantic version, 
consider different optimality criteria 
and analyze them theoretically for different DLs ranging from light-weight \EL to the more expressive \ALCI (see Table~\ref{table:complexity} for an overview).
Our contributions are three fold:
\begin{enumerate}
	\item We introduce contrastive ABox explanations along with several variants and optimality criteria.
	\item We characterize the complexity for various reasoning problems spanning five dimensions:
	variants, preference measures, types of optimality, DLs, and concept types. 
	\item We implemented a first practical method 
	and evaluated it
	on realistic ontologies.
\end{enumerate}

Full proofs of our results and details on experimental evaluation can be found in the 
\postrebuttal{technical appendix found at the end of this pdf}.
%
%
%

\begin{table}
	\centering
		\setlength{\tabcolsep}{1mm}
		\small
		\begin{tabular}{lccc}
			\toprule
			optimality & fresh ind.
			& {$\ELbot$}  &
			{$\ALC$, $\ALCI$} \\
			
			& & $\subseteq$ / $\leq$ & $\subseteq$ / $\leq$ \\
			\midrule
			\emph{diff-min}
			&
			& $\leq \Ptime^{{\text{T\ref{lem:dif-min}}}}$ / $\co\NP\text{-c}^\text{T\ref{thm:el-ground-syn-size}}$ 
			& $\EXP\text{-c}^\text{T\ref{lem:dif-min},T\ref{thm:el-ground-syn-size}}$
			\\
			\multirow{2}{*}{\emph{conf-min}}
			& {yes}
			& $\EXP\text{-c}^\text{T\ref{thm:conf-in}}$ 
			& $\co\NEXP\text{-c}^\text{T\ref{thm:conf-in}}$
			\\
			& {no}
			& $\co\NP\text{-c}^\text{T\ref{thm:elbot-no-fresh-subset}}$
			& $\EXP\text{-c}^\text{T\ref{thm:elbot-no-fresh-subset}}$
			\\
			%
			{\emph{com-max}}
			&
			& open /$\co\NP\text{-c}^\text{T\ref{thm:el-ground-syn-sim-size}}$ 
			& $\EXP\text{-c}^\text{T\ref{thm:el-ground-syn-sim-size}}$
			\\
			\bottomrule
		\end{tabular}
	\caption{
		Complexity results for verifying minimality of CEs.
	}
	
	\label{table:complexity}
\end{table}

%
%
\section{Description Logics}

We give a short exposition to the relevant DLs~\cite{DL_TEXTBOOK}.
Let $\NI$, $\NC$, and $\NR$ denote countably infinite, mutually disjoint sets of \emph{individual}, \emph{concept}
and \emph{role names}, respectively.
$\ALCI$ \emph{concepts} are concept names or built following the syntax rules in \Cref{tab:semantics}, where $R$
stands for a role name $r\in\NR$ or its inverse $r^-$.
We define further concepts as syntactic sugar: \emph{top} $\top=A\sqcup\neg A$, \emph{bottom} $\bot=A\sqcap\neg A$,
\emph{disjunction} $C\sqcup D=\neg(\neg C\sqcap\neg D)$ and \emph{value restriction} $\forall r.C=\neg\exists r.\neg C$.
A concept is in $\ALC$ if it does not use inverse roles,
in $\ELbot$ if it only uses constructs from the set $\{\top, \sqcap, \exists,\bot\}$,
and in $\EL$ if it is an $\ELbot$ concept without $\bot$.

A \emph{general concept inclusion} (GCI) is an expression of the form $C \subsum D$
for concepts
$C, D$, and a \emph{TBox} is a finite set of GCIs.
An \emph{assertion} is an expression of the form $\concept{A}(a)$ (\emph{concept assertion}) or
$r(a, b)$ (\emph{role assertion}), where $a, b \in \NI$, $\concept{A}\in \NC$  and $r \in \NR$.
An \emph{ABox} is a finite set of assertions.
Finally, a KB is a tuple $\tup{\calT,\calA}$ of a TBox~$\calT$ and an ABox $\calA$,
seen as the union $\calT\cup\calA$.
We refer to GCIs and assertions collectively as \emph{axioms}.
A TBox/KB is in \ALCI/\ALC/\ELbot/\EL if all concepts in it are.

\begin{table}[t]

	\centering
	\setlength{\tabcolsep}{3pt}
		\begin{tabular}{l  c  c }
			\toprule
			Construct           & Syntax         & Semantics \\
			\midrule
			Conjunction          & $C\sqcap D$    & $C^\Imc\cap D^\Imc$\\
			Existential restriction & $\exists R.C$ & $\{ x \mid \exists~y\in C^\calI, \tup{x,y} \in R^\Imc\}$\\
			Negation                & $\neg C$       & $\Delta^\Imc\setminus C^\Imc$\\
			\bottomrule
		\end{tabular} 
		\caption{Syntax and semantics for $\mathcal{ALCI}$ concepts. 
		}
	\label{tab:semantics}
\end{table}
The semantics of $\ALCI$ is defined in terms of interpretations.
An \emph{interpretation} $\calI$ is a tuple $\calI = (\Delta^\calI, \cdot^\calI)$,
where $\Delta^\calI$ is a non-empty set called the \emph{domain} of $\calI$, and
$\cdot^\calI$ is the \emph{interpretation function} that maps every individual name $a \in\NI$
to an element $a^\calI \in \Delta^\calI$,
every concept name $C\in N_C$ to a set $C^\calI \subseteq \Delta^\calI$, and every role name
$r\in \NR$ to a binary relation $r^\calI \subseteq  \Delta^\calI\times \Delta^\calI$.
The interpretation function is extended to inverse roles using $(r^-)^{\Imc}=\{\tup{x,y}\mid \tup{y,x}\in r^\Imc\}$
and to concepts following \Cref{tab:semantics}.

Let $C\subsum D$ be a GCI and $\calI$ be an interpretation.
Then, $\calI$ \emph{satisfies} $C\subsum D$ (denoted by $\calI\models C\subsum D$),
if $C^\calI\subseteq D^\calI$.
Similarly, $\calI$ satisfies a concept assertion $\concept{A}(a)$
if $a^\calI\in {A}^\calI$ and a role assertion $r(a, b)$ if $\tup{a^\calI , b^\calI} \in r^\calI$.
$\calI$ is a \emph{model} of $\calK$ ($\calI \models \calK$), if $\calI$ satisfies every axiom in $\calK$.
Finally, $\calK$ \emph{entails} $\alpha$ ($\calK\models\alpha$)
if $\calI\models\alpha$ for every model $\calI$ of $\calK$.
If $\calK\models C(a)$, we call $a$ an \emph{instance of $C$}.
\patrick{We use ``if'' in definitions, even though its supposed to work in both directions.
``Iff'' is instead used if we are stating a proposition.}

%

%
\section{Contrastive Explanations}\label{sec:CEs}

	In this paper, we are interested in answering contrastive questions. 
	A \emph{contrastive ABox explanation problem} (CP) is a tuple $P= \tup{\Kmc,C,a,b}$ consisting of a KB $\Kmc$, a concept $C$ and two individual names $a$, $b$, s.t. $\Kmc\models C(a)$ and $\Kmc\not\models C(b)$.
	Intuitively, a CP reads as \emph{\enquote{Why is $a$ an instance of $C$ and $b$ is not?}}.
		If $\Kmc$ and $C$ are expressed in a DL $\Lmc$, we call $P$ an \emph{\Lmc CP}.
	We call $a$ (or more generally $C(a)$) the \emph{fact} and $b$ ($C(b)$) the \emph{foil} of the CP.
	Note that in a CP, because $\Kmc\not\models C(b)$, $\Kmc$ is always consistent.

	\newcommand{\qDif}{q_\textit{diff}}


	Building upon the framework of Lipton~(\cite{lipton1990contrastive}), we aim to contrast $a$ and $b$ by highlighting the differences between the assertions that support $C(a)$ and the missing assertions that would support $C(b)$.
	Since different individuals may be related to $a$ than to $b$, we abstract away from concrete individual names and instead use \emph{ABox patterns}.
	An ABox pattern is a set $q(\vec{x})$ of ABox assertions that uses variables from $\vec{x}$ instead of individual names.
	Given a vector $\vec{c}$ of individual names with the same length as $\vec x$, $q(\vec{c})$ then denotes the ABox assertions obtained after replacing variables by individuals according to $x_i\mapsto c_i$. 
	The goal is to characterize the \emph{difference} between individuals $a$ and $b$ using an ABox pattern $\qDif(\vec{x})$, paired with two vectors $\vec{c}$ and $\vec{d}$ such that $\qDif(\vec{c})$ is entailed by the KB,
	\postrebuttal{no assertion in $\qDif(\vec{d})$ is entailed}, and adding $\qDif(\vec{d})$ to the KB would entail $C(b)$.
	In our running example, 
	 $\qDif({x,y,z})=\{\Journal(y),\hasFunding(x,z)\}$ would be such an ABox pattern,
	 where for the fact $\alice$ we have $\vec{c}=\tup{\alice,\aij,\corp}$, and for the foil
	 $\bob$ we could use $\tup{\bob,\aaai,e}$, where $e$ is fresh. 

	 \newcommand{\qCom}{q_\textit{com}}

	To fully explain the entailment, we have to also include which other facts are relevant to the entailment, which
	are things that fact and foil have in common.
	In our running example, the explanation only makes sense together with the
	\emph{commonality} $\qCom(x,y,z)=\{\publishedAt(x,y)\}$.
	Specifically, our contrastive explanations use ABox patterns $q(\vec x)=\qCom(\vec{x})\cup \qDif(\vec{x})$,
	with $\qCom(\vec{x})$ stating what holds for both instantiations, and $\qDif(\vec{x})$ what holds only for the fact.
	To avoid irrelevant assertions in $q(\vec x)$, we furthermore require that
	$q(\vec c)$ is an ABox justification in the classical sense.

	A final aspect regards how to deal with contradictions.
	Assume that in our example, instead of Axiom (b), we used
	\[
		(b')\ \exists \leads.\Group \sqsubseteq \Qualified\ .
	\]
	Most notions of abduction require the hypothesis to be consistent with the KB,
	which in the present case, due to Axiom~(c), is impossible
	if we want to entail $\Interviewed(\bob)$. For the present example, we might however still want to
	provide an explanation, for instance:
	\enquote{\emph{If AAAI was a journal and Bob lead the KR group, he would have been interviewed,
	but he cannot lead a group since he is a postdoc}}.
	This results in the following components in the contrastive explanation:
	\begin{align*}
	\qCom'=&\ \{\ \publishedAt(x,y),\Group(z)\ \},\\
	\qDif'=&\ \{\ \Journal(y),\leads(x,z)\ \},\\
	\vec{c'}=&\ \tup{\alice,\aij,\cs},\qquad\vec{d'}=\{\bob,\aaai,\cs\}
	\end{align*}
	To point out the issue with this explanation, we add a final component, the \emph{conflict set} $\Cmc$, which in this
	case would be $\{\PostDoc(\bob)\}$.
	Removing conflicts from the KB results in an alternative scenario consistent
	with the proposed explanation.
	This aligns with what are commonly called \emph{counterfactual accounts}~\cite{eiter2023contrastive}.
	Intuitively, we would want to avoid conflicts if possible, but we will see 
	later that this is
	not always desirable.

	We can now formalise our new notion of explanations.
	\begin{definition}\label{def:cex-general}
		Let $P= \tup{\Kmc,C,a,b}$ be a CP where $\Kmc=\tup{\Tmc,\Amc}$.
		A solution to $P$ (the \emph{contrastive ABox explanation}/\emph{CE})
		is a tuple $$\tup{\qCom(\vec{x}), \qDif(\vec{x}),\vec{c},\vec{d},\Cmc}$$
		of ABox patterns $\qCom(\vec{x})$, $\qDif(\vec{x})$, vectors $\vec{c}$ and $\vec{d}$ of individual names,
		and a set $\Cmc$ of ABox assertions,
		which for $q(\vec{x})=\qCom(\vec{x})\cup \qDif(\vec{x})$ satisfies the following conditions:
		\begin{enumerate}[label=\textbf{C\arabic*}]
			\item\label{itm:entailment} $\Tmc,q(\vec{c})\models C(a)$ and $\Tmc,q(\vec{d})\models C(b)$,
			\item\label{itm:fact} $\calK\models q(\vec{c})$,
			\item\label{itm:foil} $\calK \models \qCom(\vec{d})$,
			\item\label{itm:justification}
			$q(\vec{c})$ is a $\subseteq$-minimal set  satisfying \ref{itm:entailment}+\ref{itm:fact},
			\item\label{itm:conflict} $\Cmc\subseteq\Amc$ is $\subseteq$-minimal such that
			$\Tmc,(\Amc\setminus\Cmc) \cup q(\vec{d})\not\models\bot$.
		\end{enumerate}
	\end{definition}
	We call $\vec c$ the \emph{fact evidence} and $\vec d$ the \emph{foil evidence}.
	The patterns $\qCom(\vec x)$ and $\qDif(\vec x)$ will be called \emph{commonality} and \emph{difference}.
%
	Intuitively, $q(\vec{x})$ describes a pattern that is responsible for $a$ being an instance of $C$,
	with $\qCom(\vec{x})$ describing what $a$ and $b$ have in common, and $\qDif(\vec{x})$ what $b$ is lacking.
	By instantiating $\vec{x}$ with $\vec{c}$ we obtain a set of entailed assertions that entail
	$C(a)$~(\ref{itm:entailment} and \ref{itm:fact}), and by instantiating it with $\vec{d}$,
	we obtain a set of assertions that entails $C(b)$~(\ref{itm:entailment}), where $\qCom(\vec{d})$ is already provided by the present ABox~(\ref{itm:foil}),
	and $\qDif(\vec{d})$ is missing.
	Since $\qDif(\vec{d})$ can be inconsistent with the KB,
	$\Cmc$ presents the conflicts and $(\calA\setminus \calC) \cup q(\vec d)$ depicts an alternative
	consistent scenario in which $C(b)$ is entailed (\ref{itm:conflict}). To avoid
	unrelated assertions in $q$ or $\Cmc$, we require them to be minimal (\ref{itm:justification} and~\ref{itm:conflict}).
%

\begin{example}\label{ex:interview-def}
	For our running example, the CP is
	$\tup{\Kmc,\Interviewed,\alice,\bob}$.
	A CE for this CP is
	\[E_1= \tup{\qCom(x,y,z),\ \qDif(x,y,z),\ \vec{c},\ \vec{d},\ \emptyset},\]
	where $\vec c = \tup{\alice, \aij, \corp},
	\vec d = \tup{\bob, \aaai, e }$,
	\begin{align*}
		\qCom(x,y,z) = &\{\ \concept{\ \publishedAt}(x,y)\ \},\quad\text{ and }\\
		\qDif(x,y,z) = &\{\ \Journal(y), \,\hasFunding(x,z)\ \}.
	\end{align*}
	Another CE would be
	\begin{align*}
		E_2=\tup{\qCom',\qDif',\vec{c'},\vec{d'},\{\PostDoc(\bob)\}},
	\end{align*}
	with $\qCom'$ and $\qDif'$, $\vec{c'}$ and $\vec{d'}$ as described above.
\end{example}
	\subsection{Syntactic and Semantic CEs}
	For the running example, our definition also allows for the following trivial CE that has limited explanatory value:
	\[ E_t =
	 \tup{\emptyset, \{\Interviewed(x)\}, \tup{\alice}, \tup{\bob}, \emptyset}.
	\]
%
	A natural restriction to avoid this are \emph{syntactic CEs}:

	\begin{definition}[Syntactic and Semantic CEs]
		Let $P=\tup{\calK,C,a,b}$ be a CP where $\calK = \tup{\calT,\calA}$.
		A CE $$E=\tup{\qCom(\vec{x}), \qDif(\vec{x}),\vec{c},\vec{d},\Cmc}$$ for $P$ is
		called \emph{syntactic} if $\qCom(\vec{c}), \qDif(\vec{c}),\qCom(\vec{d})\subseteq \calA$, and otherwise \emph{semantic}.
	\end{definition}

%
	Syntactic explanations can only refer to what is explicit in the ABox. Semantic explanations can additionally refer to implicit information that is entailed.
	%

\newcommand{\Professor}{\concept{Prof}} 
\newcommand{\Qualify}{\concept{Qualified}} 
\newcommand{\Nominee}{\concept{Nominee}} 
\newcommand{\Offered}{\concept{Offered}} 

	\begin{example}
	\label{ex:redundancy}

		Consider the KB  $\calK = \tup{\calT,\calA}$ with
		\begin{align*}
			\calT = & \{\,\Professor\subsum \Qualify,\Qualify\sqcap \Nominee \subsum \Offered \, \} \\
			\calA = &  \{\, \Professor(\alice),\Nominee(\alice),  \Qualify(\bob) \, \}
		\end{align*}
		A syntactic explanation for $\tup{\calK,\Offered,a,b}$ is
		\[E_3=\tup{\emptyset,\ \{\Professor(x),\ \Nominee(x)\},\ \tup{\alice}, \tup{\bob}, \emptyset}.\]
		A semantic explanation can highlight the commonality:
		\[E_4=\tup{\{\Qualify(x)\}, \{\Nominee(x)\}, \tup{\alice}, \tup{\bob}, \emptyset}\]
	and thus give a more precise explanation for why Bob was not offered the job.
	(He didn't need to be a professor.)
	\end{example}

	If the CP contains a concept name as concept, a semantic explanation can always be obtained by simply using that
	concept as difference, which is why this case is more interesting for complex concepts. 
	Moreover, we can reduce
	semantic CEs to syntactic ones:
  \begin{lemma}\label{lem:materialized}
  	Let $P=\tup{\tup{\calT,\calA},C,a,b}$ be an $\mathcal L$ CP.
  	Then, one can compute in polynomial time, with access to an oracle that decides entailment for $\mathcal L$,
  	an ABox $\calA_e$ such that every semantic CE for $P$ is a syntactic CE for $P'= \tup{\tup{\calT,\calA_e},C,a,b}$ and vice versa.
  \end{lemma}
  \begin{proof}
    We simply need to add all entailed assertions of the form $A(a)$/$r(a,b)$ where $A$, $r$, $a$ and $b$ occur in the input.
  \end{proof}
  Because of \Cref{lem:materialized}, we focus on syntactic CEs for most of the paper.
  Furthermore, some reasoning problems with semantic CEs are trivial for CPs involving concept names but intractable when complex concepts are considered.
 	\subsection{Optimality Criteria}

	The minimality required in \ref{itm:justification} and \ref{itm:conflict} is
	necessary to avoid unrelated assertions in the CE. Even with these restrictions
	in place, there can be many CEs for a given CP. The idea of CEs is to choose
	the ABox pattern so that the difference is as small as possible, and the commonality
	as large as possible. Also, while we allow for conflicts, having less seems
	intuitively better. There are therefore different components one may want to optimize, and optimization may be done locally (wrt. the subset relation) or
	globally (wrt. cardinality).

	\begin{definition}[Preferred CEs]\label{def:preferences}
		Let $P$ be a CP and $E= \tup{\qCom(\vec x), \qDif(\vec x), \vec c,\vec d,\calC}$ a CE for $P$.
		Then,
		\begin{itemize}
			\item $E$ is \emph{difference-minimal} if no explanation $E'$ has difference $\qDif'(\vec x')$ and foil evidence $\vec d'$, s.t. $\qDif'(\vec d')\subset \qDif(\vec d)$. 
			\item $E$ is \emph{conflict-minimal} if there is no CE $E'$ with a conflict set $\calC'\subset\calC$.
			\item $E$ is \emph{commonality-maximal} if no CE $E'$ has commonality $\qCom'(\vec x')$ and foil evidence $\vec d'$, s.t. $\qCom(\vec d)\subset \qCom'(\vec d')$.

		\end{itemize}
		We define each of the aforementioned optimality also w.r.t. the cardinality of given sets.
	\end{definition}
	Minimizing differences aligns with the general aim of CEs---the smaller the difference,
	the easier to understand the explanation.
	Minimizing conflicts allows to deprioritize far-fetched explanations that contradict
	much of what is known about the foil---if possible, we would want to provide a CE without conflicts.
	\postrebuttal{From a practical viewpoint, difference minimality allows the smallest factual change and conflict minimality limits CEs whose difference conflicts with the known data about foil.}
	Commonality-maximality is similarly motivated, and allows to force the CE to
	be even more focussed. With commonality-maximality, we obtain interesting semantic CEs even when the concept in the CP is a concept name: in Example~\ref{ex:redundancy},
	both $E_4$ and the trivial CE using $\qDif=\{\Offered(x)\}$
	are difference-minimal, but $E_4$ is also commonality-maximal and explains the CP better.

	In Example~\ref{ex:interview-def}, $E_1$ is conflict-minimal and $E_2$ is commonality-maximal, whereas both are difference-minimal. The trivial $E_t$ is conflict- and difference-minimal, but not commonality-maximal.

\patrick{Removed subsection on decision problems --- check what needs to be adapted elsewhere.}

\paragraph{Decision Problems}
For any CP, we can always construct an arbitrary CE based on an ABox justification for the fact,
where for the foil evidence, we simply replace $a$ by $b$.
Finding CEs that are also good wrt. our optimality criteria
is less trivial.
As usual, it is more convenient to look at decision problems rather than at the computation problem, in particular at
the \emph{verification problem}:
%
%
Given a CP $P$ with CE $E$, is $E$ optimal wrt. a given criterion?
\Cref{table:complexity} gives the complexity for these problems for different DLs.
The global versions reduce to bounded versions of the existence problem:
is there an $E$ with a conflict/difference/commonality that has at most/least $n$ elements?
\ExpTime-hardness for \ALC follows in all cases by a reduction to entailment (see appendix). We discuss the other results in the following sections.

\section{Difference-Minimal Explanations}\label{sec:diff-min}





The challenge in computing and verifying difference-minimal CEs is that we cannot fix the other components:
it is possible that the difference can only be made smaller by
completely changing the other components.
%
To deal with this, we
define a maximal structure that intuitively contains all possible CEs, on which we then minimize the
different components one after the other, starting with the difference.
What it means for a structure to ``contain'' a CE is captured formally by the following definition.
We call $\tup{\qCom(\vec x), \qDif(\vec x), \vec c, \vec d, \calC}$ a \emph{candidate CE} if it
satisfies Definition~\ref{def:cex-general} except for \ref{itm:justification} and \ref{itm:conflict}. 

%


\newcommand{\pDif}{p_\textit{diff}}
\newcommand{\pCom}{p_\textit{com}}

\begin{definition}
	Let $E_p=\tup{\pCom(\vec x), \pDif(\vec x), \vec c_p, \vec d_p, \calC_p}$ and $E_q=\tup{\qCom(\vec y), \qDif(\vec y), \vec c_q, \vec d_q, \calC_q}$ be two candidate CEs for a CP $P=\tup{\calK,C,a,b}$. 
	A \emph{homomorphism} from $E_p$ to $E_q$ is a mapping $\sigma:\vec{x}\rightarrow\vec{y}$ that ensures
	$\pCom(\sigma(\vec{x}))\subseteq \qCom(\vec{y})$ and $\pDif(\sigma(\vec{x}))\subseteq \qDif(\vec{y})$.
	We say that \emph{$E_p$ embeds into $E_q$} if there is such a homomorphism and additionally $\calC_p\subseteq\calC_q$.
%
%
%
%
\end{definition}
%

Fix a CP $P= \tup{\calK, C, a, b}$.
\patrick{Add intuition?}
We define the \emph{CE super-structure} $E_m = \tup{\qCom(\vec x_m), \qDif(\vec x_m), \vec c_m, \vec d_m, \calA}$ for $P$ as follows:
\begin{itemize}
	\item $\vec{x}_m$ contains one variable $x_{a',b'}$ for every $\tup{a',b'}\in\NI(\Amc)\times\NI(\Amc)$,
	\item $\vec{c}_m$ contains $a'$ for every $x_{a',b'}$ in $\vec{x}$,
	\item $\vec{d}_m$ contains $b'$ for every $x_{a',b'}$ in $\vec{x}$,
	\item $\qCom(\vec{x}_m)=\{A(x_{a',b'})\mid A(a'),A(b')\in\Amc\}$
	$\cup\{r(x_{a_0,b_0},x_{a_1,b_1})\mid r(a_0,a_1), r(b_0,b_1)\in\Amc\}$,
	\item $\qDif(\vec{x}_m)=\{A(x_{a',b'})\mid A(a')\in\Amc,$ $b'\in\NI(\Amc)$,
	$A(b')\not\in\Amc\}$
	$\cup\{r(x_{a_0,b_0},x_{a_1,b_1})\mid r(a_0,a_1)\in\Amc,$ $b_0,b_1\in\NI(\Amc),$ $r(b_0,b_1)\not\in\Amc\}$
\end{itemize}
We set $q_m=\qCom\cup \qDif$.
The variables $\vec{x}_m$ contain all possible combinations of mapping to an individual for the foil and for the fact,
which is why they correspond to pairs of individual names.
$\qCom$ and $\qDif$ are then constructed based on the set of all assertions we can
build over these variables.
$E_m$ indeed captures all
syntactic CEs for $P$ that are defined over the signature of the input.
\begin{restatable}{lemma}{claimEmbedd}\label{claim:embed}
	Every syntactic CE for $P$ without fresh individual names embeds into $E_m$. 
\end{restatable}

%
$E_m$ is polynomial in size, and
is \emph{almost} a CE, modulo
the minimality of $q(\vec{c})$ (\ref{itm:justification}) and of $\Cmc_m$ (\ref{itm:conflict}).
Moreover, to satisfy \ref{itm:conflict}, we also need $q_m(\vec{d}_m)$ to be
consistent with $\Tmc$, which may not be the case.
To obtain a difference-minimal CE, we need to remove elements from $q_m$ to make $\Tmc, q_m(\vec{d}_m)$
consistent and $q_2(\vec{d}_m)$ subset-minimal without
violating
$\Tmc, q_m(\vec{d}_m)\models C(b)$.
The following lemma states how to safely remove elements from $q_m(\vec{d}_m)$.
For $\vec{x}\subseteq\vec{x}_m$, denote by $q_m|_{\vec{x}}(\vec{x}_m)$
the restriction of $q_m(\vec{x}_m)$ to atoms that only use variables from $\vec{x}$.
\begin{restatable}{lemma}{LemSupModels}\label{lem:restricting-x}
 Let $\vec{x}\subseteq\vec{x}_m$ be s.t 1)~$x_{a,b}\in\vec{x}$, and  2)~for every $x_{a',b'}$ in $\vec{x}_m$, we have some $x_{a'',b''} \in \vec{x}$ with $a'=a''$. Then, $\Kmc,\qDif|_{\vec{x}}(\vec{d}_m)\models C(b)$.
%
\end{restatable}

\Cref{lem:restricting-x} tells us which atoms we should keep if we want to make sure the foil remains entailed.
The following lemma tells us how to make $q_m(\vec{d_m})$ consistent with $\Tmc$:

\begin{restatable}{lemma}{ClaimConflict}\label{claim:conflict}
  Let $\vec{x}\subseteq\vec{x}_m$ be s.t.
  for every $x_{a_1,b_1}$, $x_{a_2,b_2}\in\vec{x}$,
  $a_1\neq a_2$ implies $b_1\neq b_2$. Then, $\Tmc, q_m|_{\vec{x}}(\vec{d}_m)\not\models\bot$.
\end{restatable}

Importantly, the conditions in \Cref{lem:restricting-x} and~\ref{claim:conflict} are compatible:
we can always find a vector $\vec{x}$ that is \emph{safe} in the sense that it satisfies the conditions in both lemmas, and thus ensures both consistency with $\Tmc$ and entailment of $C(b)$.
This can be used as follows to compute difference-minimal CEs, starting from the super-structure $E_m$.
\begin{enumerate}[label=\textbf{P\arabic*}]
 \item\label{p:make-consistent} Obtain a query $q^{(0)}=\qCom^{(0)}\cup \qDif^{(0)}$ from $q_m$ by removing atoms until
 $\Tmc, q^{(0)}(\vec{d}_m)\not\models\bot$. By doing so, ensure that for some safe vector $\vec{x}$,
 we have $q_m|_{\vec{x}}\subseteq q^{(0)}$. By \Cref{lem:restricting-x}, we then also have
 $\Kmc, \qDif^{(0)}(\vec{d}_m)\models C(b)$.
 \item\label{p:minimize-dif} Compute a minimal subset $\qDif(\vec{x}_m)$ of $\qDif^{(0)}(\vec{x}_m)$ s.t.
 $\Tmc, \qCom^{(0)}(\vec{d}_m)\cup \qDif(\vec{d}_m)\models C(b)$.
 This can be done in polynomial time by checking each axiom in turn.
 \item\label{p:minimize-com} To satisfy \ref{itm:justification},
 compute a minimal subset $\qCom(\vec{x}_m)$ of $\qCom^{(0)}(\vec{x}_m)$ s.t.
 $\Tmc, \qCom(\vec{c}_m)\cup \qDif(\vec{c}_m)\models C(a)$.
 This also takes polynomial time by checking each axiom in turn.
 \item\label{p:minimize-conflict} To satisfy \ref{itm:conflict}, minimize $\Cmc_m$ in the same way.
\end{enumerate}
We obtain the following theorem.
\begin{theorem}\label{lem:dif-min}
	For any DL $\mathcal L$-CPs, given an oracle that decides entailment in $\mathcal{L}$, we
	can 1)~compute a difference-minimal syntactic CE, and 2)~decide difference-minimality
	of a given syntactic CE in polynomial time.
\end{theorem}
To see why~2) holds, let $E$ be a syntactic CE. We may assume that $E$ contains no fresh individuals,
since we can always add occurrences of individuals to the KB in a way that does not affect
relevant entailments. By \Cref{claim:embed}, $E$ then embeds in $E_m$. We now apply \ref{p:make-consistent}
in the above procedure with the additional requirement that the pattern from $E$ is also contained in $q^{(0)}$, and
in \ref{p:minimize-dif}, we try to construct a subset of the difference of $E$.

This establishes our results in \Cref{table:complexity} for local difference minimality.
When minimizing the difference \emph{globally}, we lose tractability,
%
Indeed, it is $\NP$-complete to decide whether a CE exists
with the cardinality of the difference bounded by some $n\in \mathbb N$. This allows us to prove the following theorem.


\begin{restatable}{theorem}{ThmELGroundSynSize}
\label{thm:el-ground-syn-size}
	Deciding whether a given syntactic CE for a CP is difference-minimal \postrebuttal{w.r.t cardinality}
	is $\co\NP$-complete for $\EL$ and $\ELbot$, but
	$\EXP$-complete for $\ALC$.
\end{restatable}

The hardness already holds for CPs with concept names.
For complex concepts and semantic CEs, we can lift our upper bound using \Cref{lem:materialized} and prove the hardness next.
%
%

\begin{restatable}{theorem}{ThmELGroundSemSize}
\label{thm:el-ground-sem-size}
	For CPs with complex concepts,
	deciding if a given semantic CE for a CP is difference-minimal \postrebuttal{w.r.t cardinality} is 
	$\co\NP$-complete for $\EL/\ELbot$, but $\EXP$-complete for $\ALC$.
\end{restatable}
\section{Conflict-Minimal Explanations}\label{sec:conf-min}

\todo[inline]{In various places, I use $a$ both for arbitrary individuals and the individual in the CP - this has to be adapted.}


%
%
It seems natural to favor explanations with an empty or minimal conflict set.
Unfortunately, it turns out that this makes computing CEs significantly harder,
and may lead to explanations that are overall much more complex.
The reason is that avoiding conflicts may require fresh individuals in the foil evidence: 
\begin{example}
 Consider $P=\tup{\tup{\Tmc,\Amc},C,a,b}$ with
\begin{align*}
  \Tmc=&\{\ \ \exists r.\exists r.A\sqsubseteq C,\ \
  B\sqsubseteq\neg A\sqcap\forall r.\neg A\ \ \}, \\
  \Amc=&\{\ \ A(a),\ \ r(a,a),\ \ B(b)\ \ \}
 \end{align*} 
 A conflict-free syntactic CE for $P$ is
 \[
  \tup{\emptyset,\ \{\ r(x,y),\ r(y,z),\ A(z)\ \},\
  \tup{a,a,a},\ \tup{b,c,a},\  \emptyset} 
  \]
 We need the fresh individual $c$ since $b$ cannot satisfy $A$ nor
 have $a$ as successor without creating a conflict.
\end{example}


Indeed, we may even need \emph{exponentially many} of such individual names.
To show this, we reduce a problem to conflict-minimality
that has been studied under
the names \emph{instance query emptiness} \cite{DBLP:journals/jair/BaaderBL16} and
\emph{flat signature-based ABox abduction} \cite{Koopmann21a}.
We reduce from the abduction problem as it simplifies transferring size bounds.

\begin{definition}
 A \emph{signature-based (flat) ABox abduction problem} is a tuple
 $\tup{\Kmc,\alpha,\Sigma}$ of a KB $\Kmc$, an axiom $\alpha$
 (the observation) and a set $\Sigma$ of concept and role names.
 A \emph{hypothesis} for this problem is a set $\Hmc$
 of assertions
 that uses only names from $\Sigma$, and for which
 $\Kmc\cup\Hmc\not\models\bot$ and $\Kmc\cup\Hmc\models\alpha$.
\end{definition}


\begin{restatable}{lemma}{AbductionReduction}\label{lem:abduction-reduction}
 For $\Lmc\in\{\ELbot,\ALC,\ALCI\}$, let $\mathfrak{A}=\tup{\Tmc,C(b),\Sigma}$ be a signature-based ABox
 abduction problem with an $\Lmc$ TBox $\Tmc$. Then, one can construct in
 polynomial time an \Lmc-CP s.t. from every syntactic CE with empty conflict set, one can construct in polynomial time a subset-minimal hypothesis for $\mathfrak{A}$ and vice versa.
\end{restatable}
\begin{proof}[Proof sketch]
We give the idea for $\ALCI$ and provide the other
reduction in the supplemental material.
We construct a new TBox $\Tmc'$ that contains
for every CI $C\sqsubseteq D\in\Tmc $ the CI $C\sqsubseteq D\sqcup A_\bot$, where $A_\bot$ is fresh.
In addition, $\Tmc'$ contains $\exists r.A_\bot\sqsubseteq A_\bot$
and $A_\bot\sqsubseteq\forall r.A_\bot$ for every role $r$ occurring in $\Tmc$, and
the axiom $B_\bot\sqcap A_\bot\sqsubseteq\bot$, where $B_\bot$ is also fresh.
For any ABox $\Amc'$ not containing any of the
fresh names, (I) $\Tmc',\Amc'\not\models\bot$ and
(II) $\Tmc,\Amc'\models\bot$
iff for some individual name $b$, $\Tmc'\cup\Amc'\models A_\bot(b)$.
We further define
\begin{align*}
 \Amc=&\{A(a),r(a,a)\mid A\in\NC\cap\Sigma, r\in\NR\cap\Sigma\}\\
	&\cup\{A_\bot(a), B_\bot(b)\}.
\end{align*} 
The CP is now defined as $P=\tup{\tup{\Tmc',\Amc},C\sqcup A_\bot, a,b}$.
\end{proof}

\Cref{lem:abduction-reduction} also allows us to reduce the abduction problem to the
problem of computing conflict-minimal syntactic CEs.
The reason is that it is easy to extend a given CP $\tup{\Kmc,C,a,b}$ so that it has a
syntactic CE with a non-empty conflict set:
Simply add to $\Kmc$ the following axioms, where $A^*$ and $B^*$ are fresh:
$A^*(a)$, $B^*(b)$, $A^*\sqsubseteq C$, $A^*\sqcap B^*\sqsubseteq\bot$.
Now a syntactic CE can be obtained by setting $\qCom(\vec{x})=\emptyset$,
$\qDif(\vec{x})=\{A^*(x)\}$ and $\Cmc=\{B^*(b)\}$.
Consequently, we can decide the existence of a hypothesis for a
given abduction problem by computing a conflict-minimal CE and checking whether its conflict set is empty.
This now allows us to import a range of complexity results from \cite{Koopmann21a}.
To obtain matching upper bounds, we extend the construction of the CE super-structure
to now work on possible \emph{types} of individuals in the foil evidence.
Fix a CE with difference $\qDif$ and foil evidence $\vec{d}$.
Let $\Imc$ be a model of $\Kmc\cup \qDif(\vec{d})$,
and let $\Sub$ be the set of \text{(sub-)concepts} occurring in $\Kmc$ and $C$.
We assign to every $d\in\Delta^\Imc$ its \emph{type} defined as $\tp_\Imc(d)=a$ if $d=a^\Imc$ with $a$ an individual that occurs in
$\Kmc$ or $P$, and otherwise as \\[-.75em] 
 $$\tp_\Imc(d)=\{C\in\Sub\mid d\in C^\Imc\}.$$~\\[-1.25em] 
We can use the types for an arbitrary model $\Imc$ of the KB to bound the number of fresh individuals in any given CE without
introducing new conflicts. The resulting CE contains a variable $x_{c,t}$
for every individual $c$ occurring in $\Amc$ and type $t$ occurring in the range of $\tp_\Imc$.
Together with the corresponding lower bound from the abduction problem, we obtain:

\begin{restatable}{theorem}{ThmConflictFreeSize}\label{the:conflict-free-size-upper}
There exists a family of \ELbot-CPs in s.t. the size of their conflict-minimal syntactic CEs is exponential in the size of the CP.
 At the same time, every $\ALCI$-CP has a subset and a cardinality conflict-minimal
 syntactic CE whose size is  at most exponential in the size of the CP.
\end{restatable}

We can construct the set of possible types for a given CP using a type-elimination structure,
giving us a set of possible tuples for the CE in deterministic exponential time.
Based on this, we can modify our method for computing difference-minimal CEs to also construct conflict-minimal CEs. For \ALC,
using an oracle for entailment would yield membership in \TwoExpTime.
To improve this, we observe that even in models for the constructed CP, the
number of possible types is still exponentially bounded.

\begin{restatable}{theorem}{ThmConflictComplexityUpper}
	\label{thm:conf-in}
  Deciding whether a given syntactic CE for a CP is (subset or cardinality) conflict-minimal, is
  \begin{itemize}
   \item \ExpTime-complete for \ELbot-CPs,
   \item \coNExpTime-complete for \ALC- and \ALCI-CPs,
  \end{itemize}
  where the complexity only depends on the size of the CP.
\end{restatable}
A straight-forward solution to this exponential explosion
is to bound the number of individuals or to disallow fresh individuals altogether.
This immediately yields a \co\NP-upper bound for the verification, but cannot regain tractability.




\begin{restatable}{theorem}{ThmNoFreshEL}\label{thm:elbot-no-fresh-subset}
	Deciding whether a given syntactic 
	explanation without fresh individuals is (subset or cardinality) conflict-minimal is $\co\NP$-complete for $\ELbot$, but 
	\ExpTime-complete for $\ALC$ and $\ALCI$.
\end{restatable}

\section{Commonality-Maximal Explanations}


As illustrated in the extreme by the case of semantic CPs with concept names as concept,
sometimes
minimizing differences and conflicts is not sufficient, and
we want to maximize the commonality instead
to obtain a more focussed CE. This gives us an idea on how close the
foil can get to the fact.
For $\EL$, we prove that it is $\NP$-complete to decide the existence of an explanation with commonality above a
given threshold $n\in\mathbb N$. Proving the lower bound requires a different reduction as for
\Cref{thm:el-ground-syn-size}.
%
%
%
\begin{restatable}{theorem}{ThmELGroundSynSimSize}
\label{thm:el-ground-syn-sim-size}
	Deciding whether a given syntactic CE for a CP is commonality-maximal 
	is $\co\NP$-complete for $\EL/\ELbot$, but
	$\EXP$-complete for $\ALC/\ALCI$.
\end{restatable}
\section{Evaluation of a First Prototype}\label{sec:experiments}

To understand how to compute CEs in practice, we implemented a first prototype for
one of the variants~\cite{ZENODO_FILES}.
\Cref{lem:dif-min} shows that difference-minimal syntactic contrastive explanations can be computed in polynomial time, with an oracle for deciding entailment, while the other criteria are not tractable.
Semantic explanations can be reduced to syntactic ones by computing all entailed concept assertions
(\Cref{lem:materialized}), which is a standard functionality of OWL reasoning systems.
Based on these observations, we developed a prototype to compute \emph{difference-minimal syntactic CEs}.

\begin{table*}[ht]
	\centering
	\setlength{\tabcolsep}{3pt}
	\resizebox{1\columnwidth}{!}{
	\begin{tabular}{l|ccc|ccc|ccc|ccc }
\toprule
	Corpus
	& \multicolumn{3}{c|}{Signature Size}
	& \multicolumn{3}{c|}{Number of Individuals}
	& \multicolumn{3}{c|}{TBox size}
	& \multicolumn{3}{c}{ABox size}
	\\
	& avg. & med. & range
	& avg. & med. & range
	& avg. & med. & range
	& avg. & med. & range \\
	\midrule
$\ELbot$
	& 2,013 & 728 & 167 -- 7,217
	& 781 & 417 & 48 -- 3,608
	& 1,478 & 390 & 105 -- 5,653
	& 1,254 & 482 & 103 -- 8,234\\
$\ALCI$
	& 1,764 & 606 & 54 -- 7,355
	& 490  &  185 & 0 -- 5,473
	& 1,603 & 498 & 94 -- 5,286
	& 1,207 & 494 & 101 -- 9,284\\
	
	\bottomrule
		\end{tabular} 
	}
	\caption{Some details about the two corpora.}
	\label{table:meta}
\end{table*}

\begin{table*}[t]
	\centering
	\setlength{\tabcolsep}{3pt}
 	\resizebox{1\columnwidth}{!}{
	\begin{tabular}{l|cc|cc|cc|cc|cc|cc}
\toprule
	Corpus
	& \multicolumn{2}{c|}{\#CPs}
	& \multicolumn{2}{c|}{Commonality}
	& \multicolumn{2}{c|}{Difference}
	& \multicolumn{2}{c|}{Conflict}
	& \multicolumn{2}{c|}{Fresh Individuals}
	& \multicolumn{2}{c}{Duration (sec.)}

	\\
	& average & range
	& average & range
	& average & range
	& average & range
	& average & range
	& average & range \\
\midrule
	$\ELbot$
	& 35.1 & 4 -- 50
	& 0.45  & 0 -- 4
	& 1.42  & 1 -- 6
	& 0.0  &  0 -- 0
	& 0.34  & 0 -- 3
	& 2.84	&0.08 -- 386.9
\\

	$\mathcal{ALCI}$
	& 34.7 & 1 -- 50
	& 0.34 & 0 -- 7
	& 1.29 & 1 -- 5
	& 0.36 & 0 -- 9
	& 0.25 & 0 -- 5
	& 8.81 & 0.06 -- 493.6
\\
	\bottomrule
		\end{tabular}
 	}
	\caption{Results for the two corpora. ``\#CP'' states 
	the number of CPs answered (out of 50) within the timeout \postrebuttal{(10 mins)}. The other columns depict the sizes of the different components, number of fresh individuals, and computation time per explanation. }
	\label{table:results}
\end{table*}

\patrick{Check: do we give hardware details?}

\subsection{A Practical Method for Computing CEs}

To make our method for computing difference-minimal CEs practical, we refine the definition of the super structure.
Fix a CP $P=\tup{\tup{\Tmc,\Amc},C,a,b}$.
Our construction is now based on a subset $\Amc'\subseteq\Amc$, and a set $\Ibf\subseteq\NI$ of individuals to be used for the foil.
We define $E_m = \tup{\qCom(\vec x), \qDif(\vec x), \vec c_m, \vec d_m, \calC_m}$, where now
	\begin{itemize}
		\item $\vec{x}$ contains a variable $x_{c,d}$ for every $\tup{c,d}\in\NI(\Amc')\times\Ibf$,
		\item $\vec{c}_m$ uses $c$ for every $x_{c,d}$ in $\vec{x}$,
		\item $\vec{d}_m$ uses $d$ for every $x_{c,d}$ in $\vec{x}$,
		\item $\qCom(\vec{x})=\{A(x_{c,d})\mid A(c)\in\Amc',d\in\Ibf, A(d)\in\Amc\}\cup$

		$\{r(x_{c,d},x_{c',d'})\mid r(c,c')\in\Amc',d,d'\in\Ibf, r(d,d')\in\Amc\}$,
		\item $\qDif(\vec{x})=\{A(x_{c,d})\mid A(c)\in\Amc', d\in\Ibf,
		A(d)\not\in\Amc\}\cup$

		$\{r(x_{c,d},x_{c',d'})\mid r(c,c')\in\Amc', d,d'\in\Ibf, r(d,d')\not\in\Amc\}$.
	\end{itemize}
For $\Amc'=\Amc$ and $\Ibf=\NI(\Amc)$, $E_m$ is identical to the maximal CE
defined before, but too large.
%
%
Instead, for $\Amc'$ we compute the union of all justifications of $C(a)$ with an optimized implementation.
In $\Ibf$, we include individuals that are \enquote{sufficiently close} to the foil,
as well as some fresh individuals, using
\Cref{lem:restricting-x} and~\ref{claim:conflict} to ensure that $\Ibf$ is sufficiently large.
To apply~\ref{p:make-consistent}--\ref{p:minimize-conflict}, which repair and minimize the different components one after the other,
efficiently, we modified the implementation of the justification algorithm presented in~\cite{DBLP:conf/semweb/KalyanpurPHS07}
to allow computing justifications with a fixed component:
\begin{definition}
 Let $\Kmc$ be a KB, $\Kmc'\subseteq\Kmc$ and $\alpha$ an axiom s.t. $\Kmc\models\alpha$.
 A \emph{justification for $\Kmc\models\alpha$ with fixed component $\Kmc'$} is a subset-minimal
 $\Jmc\subseteq(\Kmc\setminus\Kmc')$ s.t. $\Kmc'\cup\Jmc\models\alpha$.
\end{definition}
In each step, we compute such justifications where we fix the TBox and the components that are currently not modified, which allows to speed
up the computation significantly.

Our implementation uses different optimizations for \ALCI and for \ELbot.
We implemented it in Java~8, using the OWL~API~5.1.20~\cite{OWL-API},
reasoning systems \ELK~\cite{ELK} and \HermiT~\cite{HERMIT}, as well as
the explanation library
\Evee 0.3~\cite{EVEE-LIB}, which helped us in computing unions of justifications
for \ELbot.

\subsection{Benchmark}

\newcommand{\numELOntologies}{46\xspace}
\newcommand{\numDLOntologies}{100\xspace}

\textbf{Ontologies.}
\yasir{Can we move Table~3 to the appendix and give a brief summary inline?}
We used KBs from the OWL Reasoner Competition ORE 2015~\cite{ORE_2015_ZENODO,ORE_2015_PAPER},
namely from the tracks OWL DL materialization and OWL EL materialization, restricted
respectively to \ALCI and \ELbot.
Those tracks are meant to be used for ABox reasoning, and contain KBs of varying shapes.
Some of these KBs contain all entailed assertions, which limits their use for explanations.
We step-wisely removed from each KB all entailed assertions to solve this.
We also discarded KBs with more than 10,000 axioms.
The resulting corpora contained \numELOntologies (\ELbot) and \numDLOntologies (\ALCI) KBs.
For details see \Cref{table:meta}.

\textbf{CPs.} For each KB in the corpus, we performed 5 runs and constructed 10 CPs for each run.
For each CP, we generated a random $\EL$ concept $C$ of maximum size 5 using a random walk on the
ABox starting from a randomly selected fact individual $a$. 
For this, we also considered entailed assertions.
For the foil, we selected a random individual $b$ that is not an instance of $C$ and shares at least one concept with $a$.

\subsection{Evaluation Results}

\newcommand{\inconsistentEL}{1\xspace}
\newcommand{\inconsistentDL}{3\xspace}
\newcommand{\noProblemsEL}{21\xspace}
\newcommand{\noProblemsDL}{17\xspace}
\newcommand{\memoryExceptionEL}{0\xspace}
\newcommand{\memoryExceptionDL}{0\xspace}

\newcommand{\uninterestingEL}{10\xspace}
\newcommand{\uninterestingDL}{13\xspace}

\newcommand{\interestingEL}{14\xspace}
\newcommand{\interestingDL}{67\xspace}

Some of the KBs in our corpus had to be excluded: in the \ELbot/\ALCI corpus, \inconsistentEL\/\inconsistentDL
were inconsistent,
for \noProblemsEL/\noProblemsDL no problems could be generated under our constraints.
Two more KBs from the \ALCI corpus where removed because \HermiT threw an exception on those.
Of the remaining KBs,
\uninterestingEL/\uninterestingDL KBs did not allow to produce interesting CEs:
for those KBs, every CE had an empty commonality and conflict, and exactly one axiom in the difference.
The reason was the simple structure of the ABox, which simply allowed for no more contrasting entailments,
e.g. because no role assertions were used.
We exclude those KBs in the following evaluation, and focus on the remaining \interestingEL/\interestingDL ones.

Experiments were conducted on a server with 2× Intel Xeon E5-2630 v4 20 cores, 2.2GHz CPUs, 
along with 
189 GB of available RAM running Debian 11 (Bullseye).
The Java runtime environment was OpenJDK 11.0.28. 
The results are shown in \Cref{table:results}. In general, the CEs computed tended to be simple,
even though the concepts to be explained were of size up to 5. This can again be explained with simplicity of some ABoxes:
if an individual has only one successor, even a complex concept of size 5 can only refer to those two individuals,
and consequently the CP may refer to only one fact. Nonetheless, we see that every component
of a CE is used, sometimes with several assertions, and also fresh individuals appear.
We also see that conflicts are a relatively rare occasion,
not happening at all in the \ELbot corpus, which may indicate that computing conflict-free CEs could still be feasible in practice.
What our evaluation also shows is that our current prototype takes surprisingly long to compute the answers.
Reasons include our approach
for making $q(\vec{x})$ consistent, selection of individuals, and that our construction often results in very large super-structures.
This shows potential for more dedicated methods in the future.

%
\section{Conclusion and Future Work}\label{sec:conclusion}

We introduced contrastive explanation problems and proposed CEs as a way to answer them.
It turns out that minimizing difference is tractable and feasible in practice,
while minimzing conflicts may lead to an exponential explosion. At the same time, conflicts do not
seem to happen often for realistic ontologies. In the future we want to investigate dedicated algorithms for computing CEs more efficiently. We are also exploring a variant of CEs which use quantified
variables in the fact and foil vectors. 
%
Moreover,
one can address the counting and enumeration complexity for CEs. 
To conclude, we propose CEs as a tool to contrast positive and negative query answers in ontology mediated query answering.

%
\clearpage

\section*{Acknowledgment}
We thank all anonymous reviewers for their valuable feedback. Research was partly funded by the  Deutsche Forschungsgemeinschaft (DFG, German Research Foundation), grant TRR 318/1 2021 – 438445824, the Ministry of Culture and Science of North Rhine-Westphalia (MKW NRW) within project WHALE (LFN 1-04) funded under the Lamarr Fellow Network programme, and the Ministry of Culture and Science of North Rhine-Westphalia (MKW NRW) within project SAIL, grant NW21-059D. 

\bibliographystyle{abbrv}
\bibliography{main}

@article{Miller_2021, 
  title={Contrastive explanation: a structural-model approach}, 
  volume={36}, 
  DOI={10.1017/S0269888921000102}, 
  journal={The Knowledge Engineering Review}, 
  author={Miller, Tim}, 
  year={2021}, 
  pages={e14}
 }

@inproceedings{HaifaniKTW22,
  author       = {Fajar Haifani and
                  Patrick Koopmann and
                  Sophie Tourret and
                  Christoph Weidenbach},
  editor       = {Jasmin Blanchette and
                  Laura Kov{\'{a}}cs and
                  Dirk Pattinson},
  title        = {Connection-Minimal Abduction in {$\mathcal{EL}$} via Translation to {FOL}},
  booktitle    = {Automated Reasoning - 11th International Joint Conference, {IJCAR} 2022},
  series       = {Lecture Notes in Computer Science},
  volume       = {13385},
  pages        = {188--207},
  publisher    = {Springer},
  year         = {2022},
  url          = {https://doi.org/10.1007/978-3-031-10769-6\_12},
  doi          = {10.1007/978-3-031-10769-6\_12},
  bibsource    = {dblp computer science bibliography, https://dblp.org}
}

@inproceedings{alrabbaa2022explaining,
  title={Explaining ontology-mediated query answers using proofs over universal models},
  author={Alrabbaa, Christian and Borgwardt, Stefan and Koopmann, Patrick and Kovtunova, Alisa},
  booktitle={International Joint Conference on Rules and Reasoning},
  pages={167--182},
  year={2022},
  organization={Springer}
}

@inproceedings{wei2014abduction,
  title={Abduction framework for repairing incomplete {$\mathcal{EL}$} ontologies: Complexity results and algorithms},
  author={Wei-Kleiner, Fang and Dragisic, Zlatan and Lambrix, Patrick},
  booktitle={Proceedings of the AAAI Conference on Artificial Intelligence},
  volume={28},
  number={1},
  year={2014}
}

@inproceedings{elsenbroich2006case,
  title={A case for abductive reasoning over ontologies},
  author={Elsenbroich, Corinna and Kutz, Oliver and Sattler, Ulrike},
  booktitle={Proceedings of the OWLED 2006 Workshop on OWL: Experiences and Directions, Athens, Georgia, USA, November 10-11, 2006},
  volume={216},
  pages={1--12},
  year={2006},
  organization={CEUR}
}

@inproceedings{Bienvenu08,
  author       = {Meghyn Bienvenu},
  editor       = {Gerhard Brewka and
                  J{\'{e}}r{\^{o}}me Lang},
  title        = {Complexity of Abduction in the {$\mathcal{EL}$} Family of Lightweight Description
                  Logics},
  booktitle    = {Principles of Knowledge Representation and Reasoning: Proceedings
                  of the Eleventh International Conference, {KR} 2008, Sydney, Australia,
                  September 16-19, 2008},
  pages        = {220--230},
  publisher    = {{AAAI} Press},
  year         = {2008},
  url          = {http://www.aaai.org/Library/KR/2008/kr08-022.php},
  bibsource    = {dblp computer science bibliography, https://dblp.org}
}

@article{peirce1878deduction,
  title={Deduction, induction and hypothesis: Popular Science Monthly, v. 13},
  author={Peirce, CS},
  year={1878}
}

@inproceedings{schlobach2003non,
  title={Non-standard reasoning services for the debugging of description logic terminologies},
  author={Schlobach, Stefan and Cornet, Ronald and others},
  booktitle={Ijcai},
  volume={3},
  pages={355--362},
  year={2003}
}

@book{horridge2011justification,
  title={Justification based explanation in ontologies},
  author={Horridge, Matthew},
  year={2011},
  publisher={The University of Manchester (United Kingdom)}
}

@inproceedings{BaaderPS07a,
  author       = {Franz Baader and
                  Rafael Pe{\~{n}}aloza and
                  Boontawee Suntisrivaraporn},
  editor       = {Joachim Hertzberg and
                  Michael Beetz and
                  Roman Englert},
  title        = {Pinpointing in the Description Logic {$\mathcal{EL}^+$}},
  booktitle    = {{KI} 2007: Advances in Artificial Intelligence, 30th Annual German
                  Conference on AI, {KI} 2007, Proceedings},
  series       = {Lecture Notes in Computer Science},
  volume       = {4667},
  pages        = {52--67},
  publisher    = {Springer},
  year         = {2007},
  bibsource    = {dblp computer science bibliography, https://dblp.org}
}

@inproceedings{Del-PintoS19,
  author       = {Warren Del{-}Pinto and
                  Renate A. Schmidt},
  title        = {{ABox} Abduction via Forgetting in {$\mathcal{ALC}$}},
  booktitle    = {The Thirty-Third {AAAI} Conference on Artificial Intelligence, {AAAI}
                  2019, The Thirty-First Innovative Applications of Artificial Intelligence
                  Conference, {IAAI} 2019, The Ninth {AAAI} Symposium on Educational
                  Advances in Artificial Intelligence, {EAAI} 2019},
  pages        = {2768--2775},
  year         = {2019},
  bibsource    = {dblp computer science bibliography, https://dblp.org}
}

@book{DL_TEXTBOOK,
  author       = {Franz Baader and
                  Ian Horrocks and
                  Carsten Lutz and
                  Ulrike Sattler},
  title        = {An Introduction to Description Logic},
  publisher    = {Cambridge University Press},
  year         = {2017},
  url          = {http://www.cambridge.org/de/academic/subjects/computer-science/knowledge-management-databases-and-data-mining/introduction-description-logic?format=PB\#17zVGeWD2TZUeu6s.97},
  isbn         = {978-0-521-69542-8},
  bibsource    = {dblp computer science bibliography, https://dblp.org}
}

@inproceedings{Koopmann21a,
  author       = {Patrick Koopmann},
  title        = {Signature-Based Abduction with Fresh Individuals and Complex Concepts
                  for Description Logics},
  booktitle    = {Proceedings of the Thirtieth International Joint Conference on Artificial
                  Intelligence, {IJCAI} 2021},
  pages        = {1929--1935},
  year         = {2021},
  bibsource    = {dblp computer science bibliography, https://dblp.org}
}

@book{baader2003description,
  title={The description logic handbook: Theory, implementation and applications},
  author={Baader, Franz},
  year={2003},
  publisher={Cambridge university press}
}

@inproceedings{dandl2020multi,
  title={Multi-objective counterfactual explanations},
  author={Dandl, Susanne and Molnar, Christoph and Binder, Martin and Bischl, Bernd},
  booktitle={International conference on parallel problem solving from nature},
  pages={448--469},
  year={2020},
  organization={Springer}
}

@article{verma2020counterfactual,
  title={Counterfactual explanations for machine learning: A review},
  author={Verma, Sahil and Dickerson, John and Hines, Keegan},
  journal={arXiv preprint arXiv:2010.10596},
  volume={2},
  number={1},
  pages={1},
  year={2020}
}

@book{hitzler2009foundations,
  author       = {Pascal Hitzler and
                  Markus Kr{\"{o}}tzsch and
                  Sebastian Rudolph},
  title        = {Foundations of Semantic Web Technologies},
  publisher    = {Chapman and Hall/CRC Press},
  year         = {2010},
  url          = {http://www.semantic-web-book.org/},
  isbn         = {9781420090505},
  bibsource    = {dblp computer science bibliography, https://dblp.org}
}

@inproceedings{DuWM17,
  author       = {Jianfeng Du and
                  Hai Wan and
                  Huaguan Ma},
  title        = {Practical {TBox} Abduction Based on Justification Patterns},
  booktitle    = {Proceedings of the Thirty-First {AAAI} Conference on Artificial Intelligence,
                  February 4-9, 2017, San Francisco, California, {USA}},
  pages        = {1100--1106},
  year         = {2017},
  url          = {http://aaai.org/ocs/index.php/AAAI/AAAI17/paper/view/14402},
  bibsource    = {dblp computer science bibliography, https://dblp.org}
}

@inproceedings{eiter2023contrastive,
  title={Contrastive Explanations for Answer-Set Programs},
  author={Eiter, Thomas and Geibinger, Tobias and Oetsch, Johannes},
  booktitle={European Conference on Logics in Artificial Intelligence},
  pages={73--89},
  year={2023},
  organization={Springer}
}

@article{stepin2021survey,
  title={A survey of contrastive and counterfactual explanation generation methods for explainable artificial intelligence},
  author={Stepin, Ilia and Alonso, Jose M and Catala, Alejandro and Pereira-Fari{\~n}a, Mart{\'\i}n},
  journal={IEEE Access},
  volume={9},
  pages={11974--12001},
  year={2021},
  publisher={IEEE}
}

@article{lipton1990contrastive,
  title={Contrastive explanation},
  author={Lipton, Peter},
  journal={Royal Institute of Philosophy Supplements},
  volume={27},
  pages={247--266},
  year={1990},
  publisher={Cambridge University Press}
}

@inproceedings{xai,
  author       = {Jo{\~{a}}o Marques{-}Silva and
                  Alexey Ignatiev},
  title        = {Delivering Trustworthy {AI} through Formal {XAI}},
  booktitle    = {Thirty-Sixth {AAAI} Conference on Artificial Intelligence},
  pages        = {12342--12350},
  publisher    = {{AAAI} Press},
  year         = {2022}
}

@article{dhurandhar2018explanations,
  title={Explanations based on the missing: Towards contrastive explanations with pertinent negatives},
  author={Dhurandhar, Amit and Chen, Pin-Yu and Luss, Ronny and Tu, Chun-Chen and Ting, Paishun and Shanmugam, Karthikeyan and Das, Payel},
  journal={Advances in neural information processing systems},
  volume={31},
  year={2018}
}

@inproceedings{IgnatievNA020,
  author       = {Alexey Ignatiev and
                  Nina Narodytska and
                  Nicholas Asher and
                  Jo{\~{a}}o Marques{-}Silva},
  title        = {From Contrastive to Abductive Explanations and Back Again},
  booktitle    = {AIxIA 2020 - Advances in Artificial Intelligence: XIX Int.
                  Conf. of the Italian Association for AI},
  volume       = {12414},
  pages        = {335--355},
  publisher    = {Springer},
  year         = {2020}
}

@inproceedings{DBLP:conf/kr/KoopmannDTS20,
  author       = {Patrick Koopmann and
                  Warren Del{-}Pinto and
                  Sophie Tourret and
                  Renate A. Schmidt},
  editor       = {Diego Calvanese and
                  Esra Erdem and
                  Michael Thielscher},
  title        = {Signature-Based Abduction for Expressive Description Logics},
  booktitle    = {Proceedings of the 17th International Conference on Principles of
                  Knowledge Representation and Reasoning, {KR} 2020, Rhodes, Greece,
                  September 12-18, 2020},
  pages        = {592--602},
  year         = {2020},
  url          = {https://doi.org/10.24963/kr.2020/59},
  doi          = {10.24963/KR.2020/59},
  bibsource    = {dblp computer science bibliography, https://dblp.org}
}

@article{DBLP:journals/ml/LehmannH10,
  author       = {Jens Lehmann and
                  Pascal Hitzler},
  title        = {Concept learning in description logics using refinement operators},
  journal      = {Mach. Learn.},
  volume       = {78},
  number       = {1-2},
  pages        = {203--250},
  year         = {2010},
  url          = {https://doi.org/10.1007/s10994-009-5146-2},
  doi          = {10.1007/S10994-009-5146-2},
  bibsource    = {dblp computer science bibliography, https://dblp.org}
}

@inproceedings{DBLP:conf/ijcai/FunkJLPW19,
  author       = {Maurice Funk and
                  Jean Christoph Jung and
                  Carsten Lutz and
                  Hadrien Pulcini and
                  Frank Wolter},
  editor       = {Sarit Kraus},
  title        = {Learning Description Logic Concepts: When can Positive and Negative
                  Examples be Separated?},
  booktitle    = {Proceedings of the Twenty-Eighth International Joint Conference on
                  Artificial Intelligence, {IJCAI} 2019, Macao, China, August 10-16,
                  2019},
  pages        = {1682--1688},
  publisher    = {ijcai.org},
  year         = {2019},
  url          = {https://doi.org/10.24963/ijcai.2019/233},
  doi          = {10.24963/IJCAI.2019/233},
  bibsource    = {dblp computer science bibliography, https://dblp.org}
}

@inproceedings{DBLP:conf/www/HeindorfBDWGDN22,
  author       = {Stefan Heindorf and
                  Lukas Bl{\"{u}}baum and
                  Nick D{\"{u}}sterhus and
                  Till Werner and
                  Varun Nandkumar Golani and
                  Caglar Demir and
                  Axel{-}Cyrille Ngonga Ngomo},
  editor       = {Fr{\'{e}}d{\'{e}}rique Laforest and
                  Rapha{\"{e}}l Troncy and
                  Elena Simperl and
                  Deepak Agarwal and
                  Aristides Gionis and
                  Ivan Herman and
                  Lionel M{\'{e}}dini},
  title        = {EvoLearner: Learning Description Logics with Evolutionary Algorithms},
  booktitle    = {{WWW} '22: The {ACM} Web Conference 2022, Virtual Event, Lyon, France,
                  April 25 - 29, 2022},
  pages        = {818--828},
  publisher    = {{ACM}},
  year         = {2022},
  url          = {https://doi.org/10.1145/3485447.3511925},
  doi          = {10.1145/3485447.3511925},
  bibsource    = {dblp computer science bibliography, https://dblp.org}
}

@article{DBLP:journals/jair/BaaderBL16,
  author       = {Franz Baader and
                  Meghyn Bienvenu and
                  Carsten Lutz and
                  Frank Wolter},
  title        = {Query and Predicate Emptiness in Ontology-Based Data Access},
  journal      = {J. Artif. Intell. Res.},
  volume       = {56},
  pages        = {1--59},
  year         = {2016},
  url          = {https://doi.org/10.1613/jair.4866},
  doi          = {10.1613/JAIR.4866},
  bibsource    = {dblp computer science bibliography, https://dblp.org}
}

@inproceedings{DBLP:conf/jelia/Schlobach04,
  author       = {Stefan Schlobach},
  editor       = {Jos{\'{e}} J{\'{u}}lio Alferes and
                  Jo{\~{a}}o Alexandre Leite},
  title        = {Explaining Subsumption by Optimal Interpolation},
  booktitle    = {Logics in Artificial Intelligence, 9th European Conference, {JELIA}
                  2004, Lisbon, Portugal, September 27-30, 2004, Proceedings},
  series       = {Lecture Notes in Computer Science},
  volume       = {3229},
  pages        = {413--425},
  publisher    = {Springer},
  year         = {2004},
  url          = {https://doi.org/10.1007/978-3-540-30227-8\_35},
  doi          = {10.1007/978-3-540-30227-8\_35},
  bibsource    = {dblp computer science bibliography, https://dblp.org}
}

@article{OWL-API,
  author       = {Matthew Horridge and
                  Sean Bechhofer},
  title        = {The {OWL} {API:} A {Java} {API} for {OWL} ontologies},
  journal      = {Semantic Web},
  volume       = {2},
  number       = {1},
  pages        = {11--21},
  year         = {2011},
  url          = {https://doi.org/10.3233/SW-2011-0025},
  doi          = {10.3233/SW-2011-0025},
  bibsource    = {dblp computer science bibliography, https://dblp.org}
}

@inproceedings{EVEE-LIB,
  author       = {Christian Alrabbaa and
                  Stefan Borgwardt and
                  Tom Friese and
                  Patrick Koopmann and
                  Juli{\'{a}}n M{\'{e}}ndez and
                  Alexej Popovic},
  editor       = {Ofer Arieli and
                  Martin Homola and
                  Jean Christoph Jung and
                  Marie{-}Laure Mugnier},
  title        = {On the Eve of True Explainability for {OWL} Ontologies: Description
                  Logic Proofs with {Evee} and {Evonne}},
  booktitle    = {Proceedings of the 35th International Workshop on Description Logics
                  {(DL} 2022) co-located with Federated Logic Conference (FLoC 2022),
                  Haifa, Israel, August 7th to 10th, 2022},
  series       = {{CEUR} Workshop Proceedings},
  volume       = {3263},
  publisher    = {CEUR-WS.org},
  year         = {2022},
  url          = {https://ceur-ws.org/Vol-3263/paper-2.pdf},
  bibsource    = {dblp computer science bibliography, https://dblp.org}
}

@article{ELK,
  author       = {Yevgeny Kazakov and
                  Markus Kr{\"{o}}tzsch and
                  Frantisek Simancik},
  title        = {The Incredible {ELK} --- From Polynomial Procedures to Efficient Reasoning
                  with {$\mathcal{EL}$} Ontologies},
  journal      = {J. Autom. Reason.},
  volume       = {53},
  number       = {1},
  pages        = {1--61},
  year         = {2014},
  url          = {https://doi.org/10.1007/s10817-013-9296-3},
  doi          = {10.1007/S10817-013-9296-3},
  bibsource    = {dblp computer science bibliography, https://dblp.org}
}

@article{HERMIT,
  author       = {Birte Glimm and
                  Ian Horrocks and
                  Boris Motik and
                  Giorgos Stoilos and
                  Zhe Wang},
  title        = {{HermiT}: An {OWL} 2 Reasoner},
  journal      = {J. Autom. Reason.},
  volume       = {53},
  number       = {3},
  pages        = {245--269},
  year         = {2014},
  url          = {https://doi.org/10.1007/s10817-014-9305-1},
  doi          = {10.1007/S10817-014-9305-1},
  bibsource    = {dblp computer science bibliography, https://dblp.org}
}

@misc{ORE_2015_ZENODO,
  author       = {Matentzoglu, Nicolas and
                  Parsia, Bijan},
  title        = {{ORE} 2015 Reasoner Competition Corpus},
  month        = jun,
  year         = 2015,
  howpublished = {Dataset on {Zenodo}},
  doi          = {10.5281/zenodo.18578},
  url          = {https://doi.org/10.5281/zenodo.18578},
}

@article{ORE_2015_PAPER,
  author       = {Bijan Parsia and
                  Nicolas Matentzoglu and
                  Rafael S. Gon{\c{c}}alves and
                  Birte Glimm and
                  Andreas Steigmiller},
  title        = {The {OWL} Reasoner Evaluation {(ORE)} 2015 Competition Report},
  journal      = {J. Autom. Reason.},
  volume       = {59},
  number       = {4},
  pages        = {455--482},
  year         = {2017},
  url          = {https://doi.org/10.1007/s10817-017-9406-8},
  doi          = {10.1007/S10817-017-9406-8},
  bibsource    = {dblp computer science bibliography, https://dblp.org}
}

@inproceedings{DBLP:conf/semweb/KalyanpurPHS07,
  author       = {Aditya Kalyanpur and
                  Bijan Parsia and
                  Matthew Horridge and
                  Evren Sirin},
  editor       = {Karl Aberer and
                  Key{-}Sun Choi and
                  Natasha Fridman Noy and
                  Dean Allemang and
                  Kyung{-}Il Lee and
                  Lyndon J. B. Nixon and
                  Jennifer Golbeck and
                  Peter Mika and
                  Diana Maynard and
                  Riichiro Mizoguchi and
                  Guus Schreiber and
                  Philippe Cudr{\'{e}}{-}Mauroux},
  title        = {Finding All Justifications of {OWL} {DL} Entailments},
  booktitle    = {The Semantic Web, 6th International Semantic Web Conference, 2nd Asian
                  Semantic Web Conference, {ISWC} 2007 + {ASWC} 2007, Busan, Korea,
                  November 11-15, 2007},
  series       = {Lecture Notes in Computer Science},
  volume       = {4825},
  pages        = {267--280},
  publisher    = {Springer},
  year         = {2007},
  url          = {https://doi.org/10.1007/978-3-540-76298-0\_20},
  doi          = {10.1007/978-3-540-76298-0\_20},
  bibsource    = {dblp computer science bibliography, https://dblp.org}
}

@inproceedings{DBLP:conf/esws/ChenMPY22,
  author       = {Jieying Chen and
                  Yue Ma and
                  Rafael Pe{\~{n}}aloza and
                  Hui Yang},
  editor       = {Paul Groth and
                  Maria{-}Esther Vidal and
                  Fabian M. Suchanek and
                  Pedro A. Szekely and
                  Pavan Kapanipathi and
                  Catia Pesquita and
                  Hala Skaf{-}Molli and
                  Minna Tamper},
  title        = {Union and Intersection of All Justifications},
  booktitle    = {The Semantic Web - 19th International Conference, {ESWC} 2022, Hersonissos,
                  Crete, Greece, May 29 - June 2, 2022, Proceedings},
  series       = {Lecture Notes in Computer Science},
  volume       = {13261},
  pages        = {56--73},
  publisher    = {Springer},
  year         = {2022},
  url          = {https://doi.org/10.1007/978-3-031-06981-9\_4},
  doi          = {10.1007/978-3-031-06981-9\_4},
  bibsource    = {dblp computer science bibliography, https://dblp.org}
}

@misc{ZENODO_FILES,
  author       = {Patrick Koopmann and Yasir Mahmood and Axel-Cyrille Ngonga Ngomo and Balram Tiwari},
  title        = {Can You Tell the Difference? Contrastive Explanations for ABox Entailments --- Code and Data},
  month        = nov,
  year         = 2025,
  howpublished = {Dataset on {Zenodo}},
  doi          = {10.5281/zenodo.17603219},
  url          = {https://doi.org/10.5281/zenodo.17603219},
}

\clearpage

\appendix

\section{Omitted Proof Details}

\subsection{Existence of Syntactic Explanations}
We first give an easy proof to our statement that every CP has a syntactic CE:

\begin{theorem}
 For every CP $\tup{\tup{\Tmc,\Amc},C,a,b}$, there exists a syntactic CE
 $\tup{\qCom(\vec{x}),\qDif(\vec{x}),\vec{c},\vec{d},\Cmc}$.
\end{theorem}
\begin{proof}
 Let $\calK = \tup{\Tmc,\Amc}$.
 Since $\Kmc\models C(a)$, there exists a minimal subset $\Jmc\subseteq\Amc$ s.t.
 $\Tmc,\Jmc\models C(a)$ ($\Jmc$ is an ABox justification for $\Kmc\models C(a)$.
 We let $\vec{x}$ contain a variable $x_c$ for every individual $c$ occurring in $\Jmc$.
 Then, $\vec{c}$ maps those variables $x_c$ to the corresponding individual $c$, so that
 $\Jmc=q(\vec{c})$, and $\vec{d}$ is identical except that $x_a$ maps to $b$.
 Correspondingly, $q(\vec{d})$ is the same as $q(\vec{c})$ with all occurrences of $a$
 replaced by $b$, so that indeed $\Tmc,q(\vec{d})\models C(b)$ and therefore \ref{itm:entailment}
 is satisfied. Conditions~\ref{itm:fact} and~\ref{itm:justification} are also
 satisfied since $\Jmc$ is an ABox justification. 
 Moreover, we can split $q$ into
 $\qCom$ and $\qDif$ by simply selecting from $q(\vec{d})$ the assertions
 that are contained in $\Amc$ and not contained, respectively.
 As a result, \ref{itm:foil} is satisfied .
 Finally, we pick a subset $\Cmc\subseteq\Amc$ that satisfies~\ref{itm:conflict}.
\end{proof}

The procedure used in the proof always produces a CE, but it may not be optimal
with respect to our optimality criteria, which is why we need to have a closer look
at each of them. In the case of difference-minimality, consider the CP
$P=\tup{\tup{\Tmc,\Amc},C,a,b}$ where
\begin{align*}
 \Tmc=&\{\ A\sqcap B_1\sqsubseteq C,\ A\sqcap B_2\sqsubseteq C\ \}\\
 \Amc=&\{\ A(a),\ B_1(a),\ B_2(a),\ B_2(b)\ \}
\end{align*}
An ABox justification for $C(a)$ is $\{A(a), B_1(a)\}$, based on which
we would construct the following CE:
\[
 \tup{ \emptyset,\ \{A(x),B_1(x)\},\ \tup{a},\ \tup{b},\ \emptyset}.
\]
However, this CE is neither difference-minimal nor commonality-maximal, as
illustrated by the following CE:

\[
 \tup{ \{B_2(x)\},\ \{A(x)\},\ \tup{a},\ \tup{b},\ \emptyset}.
\]

\subsection{Global Optimality and Bounded Existence}
In various places, we prove results for the verification problem for global optimality (e.g.
optimality wrt. cardinality) by using the bounded versions of the existence problem.
%
The input to this existence problem with cardinality bounds includes an additional number $n\in \mathbb N$ and asks whether there exists an explanation with $|x|\leq n$ for $x \in \{\text{difference, conflict}\}$, or with $|x|\geq n$ for $x = \text{commonality}$. The following lemma shows that this is always possible.

\begin{lemma}\label{lem:ver-to-bounded}
	Let $\mathcal L$ be any DL considered in this paper and $P$ be an $\mathcal L$-CP.
	Then deciding whether $P$ admits a CE with cardinality of the difference or conflict (resp., commonality) at most (at least) a given number $n\in\mathbb N$ has complementary complexity  to that of verifying whether a given CE $E$ for $P$ has cardinality-minimal (resp., maximal) difference or conflict (commonality).
\end{lemma}
\begin{proof}
	Fix an $\mathcal L$-CP $P=\tup{\calK, C,a,b}$ and $n\in\mathbb N$.
	We prove that membership and hardness of the bounded existence problems yield the corresponding membership and hardness of the verification with global optimality in each case.
%
%

	\textbf{For membership:}
	We reduce verification problem with global optimality to the existence with bounded components.
	Let $E$ be a CE for $P$ with difference of size $k$.
	To decide whether $E$ has cardinality-minimal difference, we can check if there exists a CE for $P$ with difference of size $k-1$ at most.
	Clearly, $E$ has cardinality-minimal difference iff $P$ does not admit a valid CE with difference of size $\leq k-1$.
	
	Similar observations apply to the case of conflicts and commonality.
	(1) Given a CE $E$ for $P$ with conflict of size $k$: $E$ has cardinality-minimal conflict iff $P$ does not admit a valid CE with conflict of size $\leq k-1$.
	(2) Given a CE $E$ for $P$ with commonality of size $k$: $E$ has cardinality-maximal commonality iff $P$ does not admit a valid CE with commonality of size $\geq k+1$.
	
%
	\textbf{For hardness:} We provide reductions in the reverse direction. 
	Precisely, we reduce existence with bounded components to verification with global optimality.
	To achieve this, given a CP $\tup{\calK, C,a,b}$ and $n\in\mathbb N$, we construct a CP $P'= \tup{\calK', C', a, b}$ and a CE $E$ for $P$ as an instance of the the verification problem.
	We consider each case separately in the following and only specify $\calK'$, the CP $P'$, and a CE $E$ for $P'$.
	
	\textbf{1: Difference-Minimality.}
	We consider fresh concepts $\{Q,A_1,\dots,A_n\}$ and  let  $\calK'=\tup{\calT',\calA'}$, where 
	\begin{align*}
		\calT' &= \calT \cup \{C \subsum Q, A_1 \sqcap\dots \sqcap A_{n} \subsum Q\}, \\
		\calA' &= \calA\cup \{A_1(a),\dots, A_{n}(a) \}.
	\end{align*}
	Moreover, we let:
	\begin{align*}
		P' &=\tup{\calK, Q,a,b} \text{ and }\\
		E &=\tup{\emptyset, \{A_1(x),\dots,A_{n}(x)\}, \tup{a}, \tup{b}, \emptyset}.
	\end{align*}
	Then $E$ is a valid CE for $P'$ and has a difference of size $n$.
	Now, the only way to have a CE $E^*$ for $P'$ with a smaller difference would be via entailing $C(a)$ and using the axiom $C\subsum Q$ to entail $Q(a)$.
	However, then $E^*$ does not use any axiom in $\calA'\setminus \calA$ and therefore is also a CE for the CP $\tup{\calK,C,a,b}$.
	As a result, $E$ is a CE for $P'$ with cardinality-minimal difference iff $P$ does not admit a CE of difference smaller than $n-1$.
	
	\textbf{2: Conflict-Minimality.}
	Here, our KB $\calK'$ uses the operator ``$\bot$'' and thus our reduction applies to $\ELbot$ and DLs beyond it. 
	This is not an issue because the case of conflict-minimal CEs is not interesting for $\EL$ as there can be no conflicts. 
	We consider fresh concepts $\{Q, A, B_1,\dots,B_n\}$ and  let $\calK'=\tup{\calT',\calA'}$, where 
	\begin{align*}
		\calT' &= \calT \cup \{C \subsum Q, A\subsum Q\}\cup\{A\sqcap B_i \subsum \bot \mid i\leq {n} \}, \\
		\calA' &= \calA\cup \{A(a), B_1(b),\dots, B_{n}(b) \}.
	\end{align*}
	Moreover, we let:
	\begin{align*}
		P' &=\tup{\calK, Q,a,b} \text{ and }\\
		E &=\tup{\emptyset, \{A(x)\}, \tup{a}, \tup{b}, \{B_1(b),\dots, B_{n}(b)\} }.
	\end{align*}
	Then, $E$ is a valid CE for $P'$ and has a conflict of size $n$.
	Now, the only way to have a CE $E^*$ for $P'$ with a smaller conflict would be via entailing $C(a)$ and using the axiom $C\subsum Q$ to entail $Q(a)$.
	However, then $E^*$ does not use any axiom in $\calA'\setminus \calA$ and therefore is also a CE for the CP $\tup{\calK,C,a,b}$.
	As a result, $E$ is a CE for $P'$ with cardinality-minimal conflict iff $P$ does not admit a CE of conflict smaller than $n-1$.
	
	\textbf{3: Commonality-Maximality.}
	We consider fresh concepts $\{Q,A_1,\dots,A_n,B\}$ and  let  $\calK'=\tup{\calT',\calA'}$, where 
	\begin{align*}
	\calT' &= \calT \cup \{C \subsum Q, A_1 \sqcap\dots \sqcap A_{n} \sqcap B \subsum Q\}, \\
		\calA' &= \calA  \cup \{A_i(a), A_i(b) \mid i\leq n \} \cup \{B(a)\} .
	\end{align*}
	Then, we let:
	\begin{align*}
		P' &=\tup{\calK, Q,a,b},  \text{ and }\\
		E &=\tup{ \{A_1(x),\dots,A_{n}(x)\}, \{B(x)\} \tup{a}, \tup{b}, \emptyset}.
	\end{align*}
	Once again, $E$ is a valid CE for $P'$ and has a commonality of size $n$.
	As before, the only way to have a CE $E^*$ for $P'$ with a larger commonality would be via entailing $C(a)$ and using axiom $C\subsum Q$.
	However, then $E^*$ does not use any axiom in $\calA'\setminus \calA$ and therefore is also a CE for $\tup{\calK,C,a,b}$.
	As a result, $E$ is a CE for $P'$ with cardinality-maximal commonality iff $P$ does not admit a CE of commonality larger than $n+1$.
\end{proof}

Lemma~\ref{lem:ver-to-bounded} allows us to only focus on proving the complexity results for the existence problem with cardinality bounds for the remainder of our paper.

\section{Difference-Minimal Explanations}

\subsection{The case of subset-minimality}

\claimEmbedd*
	\begin{proof}
		Recall that $E_m =\tup{\qCom(\vec x_m), \qDif(\vec x_m),\vec c_m, \vec d_m, \calA} $. 
		Let $E'= \tup{\pCom(\vec x), \pDif(\vec x), \vec c,\vec d, \calC}$ be a syntactic CE for $P$.
		Every variable $x_i\in\vec{x}$ corresponds to an individual name $c_i$ in $\vec{c}$
		and an individual name $d_i$ in $\vec{d}$.
		Since both $c_i$ and $d_i$ occur in $\Amc$, we can define the homomorphism
		$\sigma$ as $\sigma(x_i)=x_{{c_i},{d_i}}\in\vec{x}_m$.
		It follows now by construction that $\pCom(\sigma(\vec{x}))\subseteq \qCom(\vec{x}_m)$ and
		$\pDif(\sigma(\vec{x}))\subseteq \qDif(\vec{x}_m)$. Moreover,
		$\Cmc\subseteq\calA$ holds directly by definition of $E'$ and $E_m$.
\end{proof}

\LemSupModels*
\begin{proof}
	Let $\vec{x}$ be as in the lemma and set $q=q_m|_{\vec{x}}$.
	Moreover, let $\vec{c}$ and $\vec{d}$ be the restrictions of $\vec{c}_m$ and $\vec{d}_m$ according to $\vec{x}$.
	The condition in the lemma makes sure that that $q|_{\vec{x}}(\vec{c})=\Amc$:
	indeed, $q(\vec{c})\subseteq\Amc$ follows by construction of $E_m$ and
	$\Amc\subseteq q(\vec{c})$ because
	for every axiom $A(a')$/$r(a_1',a_2')\in\Amc$, we have $A(x_{a',b''})$/$r(x_{a_1',b_1'},x_{a_2',b_2'})\in q(\vec{x}_m)$ by the conditions on $\vec{x}$. It follows that $\Tmc,q(\vec{c}_m)\models C(a)$.
	Define a mapping $h:\vec{c}\mapsto\vec{d}$ by setting 1.~$h(a)=b$ and for every $a'\in c$,
	$h(a')=b'$ for some $x_{a',b'}\in\vec{x}$. We have that $h$ is a homomorphism from $q(\vec{c})$ into
	$q(\vec{d})$: for every $A(a')\in q(\vec{c})$, we have $A(h(a'))\in q(\vec{d})$ and for every
	$r(a_1',a_2')\in q(\vec{c}_m)$, we have $r(h(a_1'),h(a_2'))\in q(\vec{c})$. It follows that
	$\Tmc,q(\vec{d})\models C(h(a))$ and since $h(a)=b$, we have $\Tmc,q(\vec{d})\models C(b)$. Moreover, the construction of $E_m$ makes sure that $\qCom|_{\vec{x}}(\vec{d})\subseteq\Amc$, so that indeed, $\Kmc,\qDif|_{\vec{x}}(\vec{d})\models C(b)$.
\end{proof}

\ClaimConflict*
\begin{proof}
	Since $\Tmc,\Amc\not\models\bot$ and $q_m(\vec{c}_m)=\Amc$, there is a model $\Imc$ of $\Tmc, q_m(\vec{c}_m)$.
	We construct a model $\Jmc$ for $\Tmc,q_m|_{\vec{x}}(\vec{d}_m)$ by setting, for
	each $x_{a',b'}\in\vec{x}$, $(b')^\Jmc=(a')^\Imc$, and setting $X^\Jmc=X^\Imc$ for everything else.
	It is standard to verify that $\Jmc$ still satisfies all axioms in $\Tmc$.
	Moreover, since $\Imc\models q_m(\vec{c}_m)$, also $\Jmc\models q_m|_{\vec{x}}(\vec{d}_m)$.
\end{proof}


The following result appends to \Cref{lem:dif-min} and establishes the hardness for $\ALC$ when considering subset-minimality.

\begin{restatable}{theorem}{ThmALCGroundSynVerification}
\label{thm:alc-ground-syn-verification}
	For $\ALC$, deciding whether a given syntactic CE is subset difference-minimal is $\EXP$-complete. 
\end{restatable}
\begin{proof}
	The membership follows from \Cref{lem:dif-min}.

	For hardness, we consider the following reduction from instance checking for $\ALC$.
	Let $\calK= \tup{\calT,\calA}$ be a KB and $A(a)$ be an instance query.
	We let $\calK'=\tup{\calT',\calA'}$ be another KB where $$\calT'= \calT\cup \{A\sqcap B\subsum C\}$$ and
	\begin{align*}
	\calA'= \calA\cup &\{A(b),B(b)\}\\
	\cup&\{D(b),r(b,c),r(d,b)\mid D(a),r(a,c), r(d,a)\in \calA\}
	\end{align*}
	for fresh concept names $B,C$ and individual $b$.
	That is, we consider a fresh individual $b$ as a copy of $a$ ($b$ and $a$ satisfy exactly same concept and role assertions in $\calK'$).
	Then, it holds that $\calK\models A(a)$ iff $\calK'\setminus\{A(b)\}\models A(b)$, whereas $\calK'\models A(b)$ trivially since $A(b)\in \calK'$.
	This can be established easily in every model of $\calK$ via a bisimulation that maps $a$ to $b$.

	Then, we let $P= \tup{\calK', C,b,a}$ be our CP.
	Clearly, $\calK'\models C(b)$ and $\calK'\not\models C(a)$ since $B(a),C(a)\not\in \calA'$ and no axiom in $\calA$ involves the concept names $B$ or $C$.
	Moreover, we let $E= \tup{\emptyset, \{A(x),B(x)\},\tup{b},\tup{a}, \emptyset}$ be a contrastive explanation for $P$.
	Now, we prove the following claim.

	\begin{claim}
		$\calK\models A(a)$ iff $E$ is not syntactic (subset) difference-minimal CE for $P$ in $\calK'$.
	\end{claim}
	\begin{claimproof}

		($\Rightarrow$) If $\calK\models A(a)$.
		Then, we construct an ABox pattern $\qCom(\vec x)$ from an ABox justification $J\subseteq \calA$ for $A(a)$.
		Then, let $\qDif(y) = \{B(y)\}$, which together with the TBox axiom $A\sqcap B\subsum C$ results in a CE $\tup{\qCom(\vec x,y), \qDif(\vec x, y), \vec c, \vec d,\emptyset}$ for $P$.
		Clearly, there are no conflicts since $\calK$ is consistent, and therefore $\calC=\emptyset$.
		Moreover, since $\qDif(\vec d) = B(a)$, thus $E$ is not a difference-minimal CE for $P$ in $\calK'$.

		($\Leftarrow$) If $\calK\not\models A(a)$, then clearly $\calK'\setminus \{A(b)\}\not\models A(b)$.
		As a result, the only way to entail $C(b)$ is via axioms $\{A(b), B(b)\}$.
		This yields $E$ as a unique CE 
		for $P$ in $\calK'$. 
		This holds since $\qDif(\vec d) = \{A(a), B(a)\}$ and there is no syntactic ground CE for $P$ in $\calK'$ with difference set of size one.
	\end{claimproof}
	This completes the proof to our theorem since the reduction can be obtained in polynomial time.
\end{proof}

The following corollaries regarding the complexity of semantic CEs can be derived using~\Cref{lem:materialized}.

\begin{corollary}
	For any DL $\mathcal L$ and $\Lmc$-CP, a semantic subset difference-minimal CE can be computed in
	polynomial time with access to an oracle that decides entailment for $\mathcal L$.
\end{corollary}


\begin{corollary}\label{cor:alc-ground-sem-verification}
	For $\ALC$, deciding whether a given a semantic CE is subset difference-minimal is $\EXP$-complete. 
\end{corollary}

\subsection{The case of cardinality-minimality}

We now turn towards the verification of globally optimal explanations (with cardinality).
\yasir{there is a typo: missing `cardinality' before diff-min.}

\ThmELGroundSynSize*

\begin{proof}
	We complete the proof in two parts. 
	\Cref{thm:el-ground-syn-size} proves the claim for $\EL$ and $\ELbot$, whereas \Cref{thm:alc-ground-syn-size} covers the case for $\ALC$.
	Moreover, both theorems prove the results for size bounded CEs and the claim of the present theorem follows by using \Cref{lem:ver-to-bounded}.
\end{proof}

\begin{lemma}
	\label{thm:el-ground-syn-size}
	For $\EL$ and $\ELbot$, deciding the existence of a syntactic CE with
	difference of size at most $n$ is \NP-complete.
\end{lemma}

\begin{proof}
	The membership is easy since one can guess ABox patterns $\qCom(\vec x)$, $\qDif(\vec x)$ with $|\qDif|\leq n$, $|\qCom|\leq |\calA|$ and vectors $\vec c$, $\vec d$, such that $\tup{\qCom(\vec x), \qDif(\vec x), \vec c, \vec d, \emptyset}$ is a valid CE.
	The verification requires polynomial time. This gives membership in $\NP$ for $\EL$.
	Notice that for $\ELbot$, an explanation might trigger inconsistency, however, the conflict set $\calC$ can still be computed in polynomial time since consistency for $\ELbot$ remains $\Ptime$-checkable.
	That is, check whether $\calT,\calA\cup  \qDif(\vec d)\models \bot$, if yes, we keep removing elements from $\calA$ to a set $\calC$ until the entailment becomes false.
	This gives the set $\calC$ in polynomial time. Hence, the overall complexity is in $\NP$ for $\ELbot$.

	For hardness in $\EL$, we reduce from the hitting set problem using a reduction inspired by Baader et al.~\cite[Thm.~3]{BaaderPS07a}.
	Their reduction yields a TBox with the aim to decide the existence of a justification of a given size.
	We extend it to add ABox assertions since in our case the size of a CE depends on the size of ABox.
	Let $G =(V,E)$ be a hypergraph with nodes $V$ and edges $E= \{e_1,\dots, e_k\}$ where $e_i\subseteq V$ for each $i\leq k$. Let $n\in \mathbb N$, the task is to determine whether there is a set $S\subseteq V$ such that $|S|\leq n$ and $S \cap e_i \neq \emptyset$ for each $i \leq k$.
	The set S is called a hitting set of $G$.

	In our reduction, we use the set of concept names $\{P_v \mid v\in V\} \cup \{C, Q_1,\dots,Q_k\}$.
	Given, $G=(V,E)$ with edges $E= \{e_1, \dots, e_k\}$, we let $\calK\dfn \tup{\calT,\calA}$ be a KB specified as follows:
	\begin{align*}
		\calT &= \{ P_{v}\subsum Q_i \mid \text{ if } v\in e_i \text{ for } i\leq k \} \\
		& \quad \cup \{ Q_1\sqcap \dots \sqcap Q_k \subsum C\},\\
		\calA &= \{P_{v}(a) \mid v\in V \},
	\end{align*}
	where $a$ is an individual. 
	Then, we let $b\neq a$ be a fresh individual and define $P= \tup{\calK, C, a,b}$ a CP.

	\begin{claim}\label{claim:card-q2-np}
		$G$ admits a hitting set $S$ with $|S|\leq n$ iff there are ABox patterns $\qCom\cup \qDif$ over $\calK$ with $|\qDif|\leq n$ and vectors $\vec c, \vec d$, such that $\tup{\qCom(\vec x),\qDif(\vec x),\vec c, \vec d, \emptyset}$ is a CE for $P$.
	\end{claim}
	\begin{claimproof}
		($\Rightarrow$)
		Suppose $S$ is a hitting set for $G$ of size $m\leq n$.
		We, consider the ABox assertions $\{P_{v}(a) \mid v\in S \}$ and their corresponding ABox pattern $\{P_{v}(x) \mid v\in S\}$.
		Moreover, we take
		\begin{itemize}
			\item $\qCom(x) = \emptyset$,
			\item $\qDif(x)=\{P_v(x)\mid v\in S\}$,
			\item $ \vec c = \tup{a}$, and
			\item $ \vec d = \tup{b}$.
		\end{itemize}
		Clearly, $|\qDif(x)|\leq n$ since $|S|\leq n$.
		We now prove that $\tup{\qCom(x), \qDif( x), \vec c, \vec d, \emptyset}$ is a valid contrastive explanation for $\tup{\calK, C, a,b}$.
		We let $q=\qCom\cup\qDif$.

		(I.) $\calT, q(\vec c)\models C(a)$ and  $\calT, q(\vec d)\models C(b)$. Since $S\cap e_i \neq \emptyset$ for each $i\leq k$, we can find ABox assertions $P_{v}(a) \in q(c)$ such that $v\in e_i$ and therefore $\calT, \{P_{v}(a)\} \models Q_i(a)$ for each $i\leq k$.
		Then, $\calT, q(c)\models Q_1\sqcap\dots\sqcap Q_k(a)$ and consequently, $\calT, q(\vec c)\models C(a)$ follows.
		The case for $\calT, q(\vec d)\models C(b)$ follows analogously.

		The remaining conditions:
		(II.) $q(\vec c) \subseteq \calA$, and
		(III.) $\qCom(\vec d) \subseteq \calA$ are also easy to observe, thus proving the claim in this direction.
		Observe that we do not have the check for minimality since the existence of a candidate CE with difference bounded by $n\in\mathbb N$ implies the existence of a CE with the same or a smaller difference set.

		($\Leftarrow$).
		Let $\tup{\qCom(\vec x), \qDif(\vec x), \vec c, \vec d, \emptyset}$ be an explanation for $\tup{\calK, C, a,b}$ with $|\qDif(\vec d)|\leq n$. 
		Then, we construct a hitting set $S$ for $G$ from $q(\vec d)$ in the following
		Notice first that, to achieve $\calT, q(\vec d)\models C(b)$ requires $\calT, q(\vec d)\models Q_i(b)$ for each $i\leq k$ which in turn requires $\{P_v(b), P_v \subsum Q_i\}$ such that $v\in e_i$.
		However, note that $P_v(b)\not\in \calA$ for any $v\in V$ and therefore $P_v(b)\in \qDif(\vec d)$ must be the case for the entailment to hold.
		Now, $|\qDif(\vec d)|\leq n$ implies that there are at most $n$ elements $v\in V$ such that $P_v(b)\in \qDif(\vec d)$.

		We let $S = \{v \mid P_v(b) \in \qDif(\vec d)\}$.
		Clearly, $|S|\leq n$.
		Then, $S$ is a hitting set since $\calT, q(\vec d)\models Q_i(b)$ for each $i\leq k$.
		Assume towards a contradiction that $S\cap e_j = \emptyset$ for some $j\leq k$.
		Now, we look at the entailment $\calT,  q(\vec d) \models Q_j(b)$.
		But this is only possible if $P_v(b)\in q(\vec d)$ for some $v\in e_j$.
		As a result, we get $v\in S$ due to the way $S$ is defined.
		This leads to a contradiction since $v\in S\cap e_j$.
		This completes the correctness of our claim.
	\end{claimproof}
	We conclude by observing that the reduction can be achieved in polynomial time.
\end{proof}

\begin{restatable}{lemma}{ThmALCGroundSynSize}
\label{thm:alc-ground-syn-size}
	For $\ALC$, deciding the existence of a syntactic or semantic CE with
	difference of size at most $n$ is \ExpTime-complete.
\end{restatable}
\begin{proof}
	The membership follows, since one can try all possible explanations with difference of size at most $n$ in exponential time.

	For hardness, we reuse the reduction from the proof of Theorem~\ref{thm:alc-ground-syn-verification}.
	Precisely, we use the same CP as there.
	The correctness also remains the same, except we prove the following claim instead.
	\begin{claim}
		$\calK\models A(a)$ iff $\tup{\calK', C,a,b}$ admits a syntactic explanations with difference of size one.
	\end{claim}
	The claim proof also follows using the same argument as before. This results in the stated completeness. \qedhere
\end{proof}

\ThmELGroundSemSize*

\begin{proof}
	The case for $\ALC$ is already covered by \Cref{thm:alc-ground-syn-size}.
	We complete the proof to this theorem by proving  the claim 
	for $\EL$ and $\ELbot$ in \Cref{thm:el-sem-size}. 
	The theorem statement follows by using \Cref{lem:ver-to-bounded}.
\end{proof}

\begin{lemma}
	\label{thm:el-sem-size}
	For CPs with complex concept, deciding the existence of a semantic CE with
	difference of size at most $n$ is \NP-complete for $\EL$ and $\ELbot$.
\end{lemma}

\begin{proof}
	The membership follows due to \Cref{lem:materialized} and \Cref{thm:el-ground-syn-size}.
	
	For hardness in $\EL$, we reuse the reduction from the proof of \Cref{thm:el-ground-syn-size} where the concept name $C$ in the CP $P$ is replaced by the complex concept $Q_1\sqcap\dots\sqcap Q_k$.
	Now, we reprove Claim~\ref{claim:card-q2-np} considering semantic CEs.
	Notice that a syntactic explanation is also a semantic explanation.
	Therefore, every hitting set for $G$ of size at most $n$ results in an explanation with $|\qDif|\leq n$ (direction ``$\Longrightarrow$'' in Claim~\ref{claim:card-q2-np}).
	Conversely, we prove that if the CP constructed in the proof of \Cref{thm:el-ground-syn-size} admits a semantic explanation with $|\qDif|\leq n$, then $G$ has a hitting set $S$ with $|S|\leq n$.
	Suppose $E= \tup{\qCom(\vec x), \qDif(\vec x), \vec c, \vec d,\emptyset}$ be a CE.
	Since $\calA$ does not include any assertion involving $b$, every assertion involving $b$ must be included in $\qDif(\vec d)$.
	Observe that, for each $i\leq k$, the entailment $\calT, q(\vec d) \models Q_i(b)$ is true iff either $Q_i(b)\in q(\vec d)$, or $P_v(b)\in q(\vec d)$ for some $v\in e_i$.
	This holds since a CE can not contain complex patterns (e.g., $(Q_1\sqcap Q_2)(x)$ or $(P_v\sqcap P_{v'})(x)$.
	As a result, for each $i\leq k$, $\qDif(\vec d)$ contains either $Q_i(b)$ or $P_v(b)$ where $b\in e_i$.
	Moreover, $Q_i(b)\not\in \qCom(\vec d)$ as $\calT,\calA \models \qCom(\vec d)$ by definition.
	Therefore, $\qDif(\vec d)$ has the form $ \{Q_j (b) \mid j\in k\} \cup \{P_v(b) \mid v\in V\}$.
	By assumption that $E$ is a CE, $\calT, \qDif(\vec d)\models Q_i(b)$ for each $i\leq k$.
	We let $S=\{v \mid v\in e_j \text{ if } Q_j(b) \in \qDif(\vec d)\} \cup \{v \mid P_v(b)\in \qDif(\vec d) \} $.
	Then $|S|\leq n$ as $|\qDif(\vec d)|\leq n$.
	$S$ is a hitting set, since (1) $S\cap e_j \neq \emptyset $ for each $j$ with $Q_j\in \qDif$ and (2) $S\cap e_j \neq \emptyset $ as $\calT, q(\vec d)\models Q_i(b)$ if $Q_i\not\in \qDif$.
	This proves the claim and completes the proof to the theorem.
\end{proof}

\section{Conflict-Minimal Explanations}
We group theorem proofs for readability and distinguish the case with and without fresh individuals.
\subsection{The case with fresh individuals}

\AbductionReduction*
\begin{proof}
  The main text only contains the construction for \ALCI. For \ELbot, we have to use a slightly different construction: $\Tmc'$ is then obtained from $\Tmc$ by replacing $\bot$ everywhere
  by $A_\bot$, and by adding the axioms $A_\bot\sqcap B_\bot\sqsubseteq\bot$ and
  $\exists r.A_\bot\sqsubseteq A_\bot$ for every role name occurring in the input. We also have to change the concept in the CP, since we cannot use disjunctions in \ELbot. Instead of using $C\sqcup A_\bot$ as concept, we introduce a fresh concept
  name $A_C$, for which we add the GCIs $C\sqsubseteq A_C$ and $A_\bot\sqsubseteq A_C$.

  To keep the following simpler, we focus in the following on the case for \ALCI --- the proof
  for \ELbot is very similar (see footnotes).

 We first prove our two claims (I) and (II) from the main text hold for any ABox $\Amc'$.
 \begin{itemize}
  \item[(I)] To see that $\Tmc',\Amc'\not\models\bot$, let $\Imc$ be a model of $\Amc'$,
  and extend it to a model $\Jmc$ of $\Tmc'$ by setting $A_\bot^\Jmc=\Delta^\Imc$. This
  way, all CIs are satisfied, so that $\Jmc\models\Tmc',\Amc'$.
  \item[(II)] Assume $\Tmc',\Amc'\models A_\bot(b)$, and assume that $\tup{\Tmc,\Amc'}$ is consistent.
  There is then a model $\Imc$ of $\tup{\Tmc,\Amc'}$, and since $A_\bot$ does not occur in either $\Tmc'$
  or $\Amc'$, we can assume that $A_\bot^\Imc=\emptyset$.
  Since $\Imc$ is a model of $\Tmc$,
  in which for every GCI $C\sqsubseteq D\in\Tmc$, $C^\Imc\subseteq D^\Imc$, which means that
  also, $\Imc\models C\sqsubseteq D\sqcup A_\bot$. Moreover, since $A_\bot^\Imc=\emptyset$,
  also $(\exists r.A_\bot)^\Imc=\emptyset$ and $(\exists r^-.A_\bot)^\Imc$, which implies that
  $\Imc\models\exists r.A_\bot\sqsubseteq A_\bot$ and $\Imc\models\exists r^-.A_\bot$.
  We obtain that $\Imc$ is a model of $\tup{\Tmc',\Amc'}$
  with $(A_\bot)^\Imc=\emptyset$, which contradicts our assumption that
  $\Tmc,\Amc'\models A_\bot(b)$.

 For the other direction, assume $\tup{\Tmc,\Amc'}\models\bot$. We already observed in (I) that
 $\tup{\Tmc',\Amc'}$ is consistent, and thus has a model $\Imc$. Let $\Imc$ be any such model.
 We may assume wlog. that every domain element in $\Imc$ is connected to some element of the ABox
 via a chain of role-successors, where those chains may involve
 role names and their inverses. This is wlog. since the GCIs in \Tmc' can refer to other elements
 only via role restrictions. Furthermore, we must have $A_\bot\neq\emptyset$, since otherwise
 $\Imc$ would also be a model of $\Tmc,\Amc'$, which would contradict our initial assumption.
 Let $d\in(A_\bot)^\Imc$. Because $\Imc\models A_\bot\sqsubseteq\forall r.A_\bot$,
 every $r$-successor of $d$ satisfies $A_\bot$, and because
 $\Imc\models\exists r.A_\bot\sqsubseteq A_\bot$, every $r$-predecessor of $d$ satisfies
 $A_\bot$. By applying this argument inductively, we obtain that every element that is
 connected to $d$ is an instance of $A_\bot$, including, by our assumption on $\Imc$,
 some element $b^\Imc$ where $a$ occcurs in the ABox.
 \footnote{
 Note that in the case of \ELbot, the argument is much simpler: the only way to produce
 a contradiction is by explicitly using $\bot$ in the TBox, which in $\Tmc'$ would translate to some individual satisfying $A_\bot$. The axioms $\exists r.A_\bot\sqsubseteq A_\bot$ then propagate this concept name back towards some individual in the ABox.
 }
 From this argument it follows that
 for every model $\Imc$ of $\tup{\Tmc',\Amc'}$, we have
 $\Imc\models A_\bot(b)$ for some $b\in\NI$. It remains to show that
 also $\Imc\models A_\bot(b)$ for the \emph{same} $b\in\NI$ for every model $\Imc$ of $\Tmc',\Amc'$.
 \footnote{
 Again, the situation is easier in \ELbot, where we do not have to argue over different models,
 since \ELbot has the canonical model property, e.g. there exists a single model that can be
 embedded into any other model, and this includes the instance of $A_\bot$.
 }

 Assume we have another model $\Jmc$ s.t. $\Jmc\not\models A_\bot(b)$. There must then be another
 individual $c$ s.t. $\Jmc\models A_\bot(c)$. Moreover, $b^\Imc$ and $c^\Imc$ are not connected
 to each other in $\Jmc$, since otherwise $\Jmc\models\ A_\bot(b)$. This also means that $b$ and
 $c$ are not connected to each other in $\Amc$. Because of this, we can construct a model
 $\Jmc'$ of $\tup{\Tmc',\Amc'}$ based on $\Imc$ and $\Jmc'$ s.t.
 $\Jmc'\not\models A_\bot(b)$ and $\Jmc'\not\models A_\bot(c)$. By applying this argument repeatedly,
 we would be able to build a model in which no individual satisfies $A_\bot$, which we have
 already proven to be impossible. It follows that there exists some $b\in\NI$ s.t. for all
 models $\Imc$ of $\Tmc',\Amc'$, we have $\Imc\models A_\bot(b)$.
 \end{itemize}

 To show that the CP constructed in the main text is indeed a valid CP,
 we need to prove that for $\Kmc=\tup{\Tmc',\Amc}$, we have $\Kmc\models C\sqcup A_\bot(a)$
and $\Kmc\not\models C\sqcup A_\bot(b)$. The former follows since $A_\bot(a)\in\Amc$,
and the latter follows from the fact that the individual name $b$ only occurs in one
concept assertion that is not connected to the rest of the ABox, and the earlier
observation that $\Kmc\not\models\bot$.

We show that every subset-minimal hypothesis $\Hmc$ for $\mathfrak{A}$ can be transformed in
polynomial time into a syntactic solution for the CP with empty conflict set.
Let $\Hmc$ be such a hypothesis.
We construct the solution $E=\tup{\qCom(\vec{x}), \qDif(\vec{x}), \vec{c}, \vec{d},\emptyset}$ as follows:
\begin{itemize}
\item $\vec{x}$ contains a variable
$x_c$ for every individual name $c$ occurring in $\Hmc$;
\item $\qCom(\vec{x})=\emptyset$;
\item $\qDif(\vec{x})$ contains all assertions from $\Hmc$ with individual names $c$ replaced by the corresponding variable $x_c$;
\item $\vec{c}$ is a sequence of $a$'s; and
\item $\vec{d}$ associates each variable $x_c$ with the corresponding individual name $c$.
\end{itemize}

\begin{claim}
	\emph{$\Tmc',\qDif(\vec{c})\models C\sqcup A_\bot(a)$.}
\end{claim}
\begin{claimproof}
Assume $\Tmc',\qDif(\vec{c})\not\models C\sqcup A_\bot(a)$.
Then, there is a model $\Imc$ of $\Tmc',\qDif(\vec{c})$ s.t. $a^\Imc\not\in (C\sqcup A_\bot)^\Imc$.
We may furthermore assume that every domain element in $\Imc$ is connected to $a^\Imc$ via role successors and predecessors. Since $a^\Imc\not\in A_\bot^\Imc$, due to the axioms $\exists r.A_\bot\sqsubseteq A_\bot$ and $A_\bot\sqsubseteq \forall r.A_\bot$, we then must have $A_\bot^\Imc=\emptyset$.

We transform $\Imc$ into a model $\Jmc$ of $\Tmc\cup\Hmc$:
\begin{itemize}
 \item $\Delta^\Jmc=\Delta^\Imc\cup\{c^\Jmc\mid c\in\NI(\Hmc)\}$
 \item for all $A\in\NC$: $A^\Jmc= A^\Imc\cup\{c^\Jmc\mid a^\Imc\in A^\Imc, c\in\NI(\Hmc)\}$
 \item for all $r\in\NR$:

 $r^\Jmc=r^\Imc\cup\{(c^\Jmc,d)\mid (a^\Imc,d)\in r^\Imc, c\in\NI(\Hmc)\}$

                                       $\cup\{(d,c^\Jmc)\mid (d,a^\Imc)\in r^\Imc, c\in\NI(\Hmc)\}$

                                       $\cup\{(c_1^\Jmc,c_2^\Jmc)\mid (a^\Imc,a^\Imc)\in r^\Imc, c_1,c_2\in\NI(\Hmc)\}$
\end{itemize}
From the fact that $A_\bot^\Imc=\emptyset$ and $\Imc\models\Tmc'$, we directly obtain $\Jmc\models\Tmc$. Furthermore, for every $A(c)$/$r(c,d)\in\Hmc$, we have $A(a)$/$r(a,a)\in \qDif(\vec{c})$, so that the construction also ensures that $\Jmc\models\Hmc$.
For the same reason, since $a^\Imc\not\in C^\Imc$, we also have $b^\Jmc\not\in C^\Jmc$, which contradicts that $\Hmc$ was a hypothesis for the abduction problem to begin with.
As a conclusion, we obtain that our initial assumption was wrong, and that $\Kmc\models C\sqcup A_\bot(a)$.
\end{claimproof}
It follows now from the construction, the proof in the previous claim, and the fact that $\Hmc$ is a hypothesis that $E$ is a syntactic contrastive explanation, namely that for $q(\vec{x})=\qCom(\vec{x})\cup \qDif(\vec{x})$
\begin{enumerate}
	\item $\Tmc',q(\vec{c})\models C\sqcup A_\bot(a)$ (Claim~1) and $\Tmc',q(\vec{d})\models C(b)$ (because $\Tmc\cup\Hmc\models C(a)$);
	\item $q(\vec{c})\subseteq\Amc$ (by construction),
	\item $\qCom(\vec{d})\subseteq\Amc$ (trivially, since $\qCom(\vec{d})=\emptyset$),
	\item $q(\vec{c})$ is minimal (because $\Hmc$ is)
	\item $\Kmc,\qDif(\vec{d})\not\models\bot$ (since $\Tmc\cup\Hmc\not\models\bot$).
\end{enumerate}

It remains to show that also
every syntactic CE for the CP can be transformed in polynomial time into a subset-minimal hypothesis for $\mathfrak{A}$.

Let $\tup{\qCom(\vec{x}),\qDif(\vec{x}),\vec{c},\vec{d},\emptyset}$ be a syntactic CE for the CP. We show that $\Hmc=\qDif(\vec{d})$ is a hypothesis for
$\mathfrak{A}$. By construction, $q(\vec{c})\subseteq\Amc$ and the fact that $a$ is not connected to $b$, we obtain that $q(\vec{c})$ contains only concept and role names from $\Sigma$, so that $\sig(\Hmc)\subseteq\Sigma$ (\emph{first condition for hypotheses}). Because the CE is conflict free, $\Kmc,\Hmc\not\models\bot$. Furthermore, $B_\bot(b)\in\Amc$ by construction, which together with the TBox axioms ensures
that no individual in $\Hmc$ that is connected to $b$ satisfies $A_\bot$. It follows that $\Tmc',\Hmc\not\models A_\bot(c)$ for any individual name in $\Hmc$, so that by our property \textbf{(II)}, $\Tmc,\Hmc\not\models\bot$ (\emph{second condition for hypotheses}).
We already observed that $\Tmc,\Hmc\not\models A_\bot(b)$. Because of our CP however, $\Tmc',\Hmc\models C\sqcup A_\bot(b)$, so that we obtain $\Tmc',\Hmc\models C(b)$. It is an easy consequence that $\Tmc,\Hmc\models C(b)$ (\emph{third and last condition for hypotheses}).
\end{proof}

\ThmConflictFreeSize*
\begin{proof}
The lower bounds follow from \Cref{lem:abduction-reduction} and the corresponding result
for flat signature-based ABox abduction~\cite{Koopmann21a}. We thus only need to show the upper bound.

 Let $P=\tup{\Kmc,C,a,b}$ be the CP, where $\Kmc=\tup{\Tmc,\Amc}$, and $E=\tup{\qCom(\vec{x}),\qDif(\vec{x}),\vec{c},\vec{d},\Cmc}$ be a syntactic CE for it, where $\Cmc$ is subset or cardinality minimal. We are going to transform $E$ so that the length of $\vec{x}$ is bounded by an exponential on the size of $P$, while the conflict set $\Cmc$ remains the same.

 Most importantly, our construction has to preserve~\ref{itm:conflict}.
 Let $\Imc$ be a model of $\tup{\Tmc,(\Amc\setminus\Cmc)\cup q(\vec{d})}$, where $q$ is again the union of $\qCom$ and $\qDif$,
 and let $\Sub$ contain the set of (sub-) concepts occurring in $\Kmc$ and $C$. We assign to every $d\in\Delta^\Imc$ its \emph{type} defined as $\tp_\Imc(d)=c$ if $d=c^\Imc$ with $c$ an individual that occurs in $\Kmc$ or $P$, and otherwise as
 $$\tp_\Imc(d)=\{C\in\Sub\mid d\in C^\Imc\}.$$

 We use these types to construct an exponentially bounded CE. The new vector $\vec{x'}$ contains a variable $x_{c,t}$ for every individual $c$ occurring in $\Amc$ and every type $t$ occurring in the range of $\tp_\Imc$. This ensures that the length of $\vec{x'}$ is exponentially bounded. The new vectors $\vec{c'}$ and $\vec{d'}$ are constructed as follows: $\vec{c'}$ assigns to each variable $x_{c,t}$ in $\vec{x}$ the individual name $c$, and $\vec{d'}$ assigns to each variable $x_{c,t}$ the individual name $d_t$, which is the individual $d$ if $t=d$ is an individual name, and otherwise a fresh individual unique to the type $t$.
 The new ABox pattern $q'(\vec{x})$ contains an assertion $A(x_{c,t})$ for every assertion $A(x)\in q(\vec{x})$ s.t. $x$ is mapped in $\vec{c}$ to $c$  and in $\vec{d}$ to some $d$ s.t. $\tp_\Imc(d^\Imc)=t$.
 Our construction ensures that $q'(\vec{c'})$ and $q(\vec{c})$ contain the same atoms, since
 every atom $A(x_{c,t})\in q'(\vec{x'})$ corresponds to an atom $A(c)\in q(\vec{x})$.
  $\qCom'(\vec{x'})$ and $\qDif'(\vec{x'})$ are now constructed from $q'(\vec{x'})$ in the obvious way. The following conditions for syntactic CEs are easy to establish:
 \begin{enumerate}
  \item $\Tmc,q'(\vec{c'})\models C(a)$, because $q(\vec{c})\subseteq q'(\vec{c'})$ and
  $\Tmc,q(\vec{c})\models C(a)$ (first part of \ref{itm:entailment});
  \item $q'(\vec{c'})\subseteq\Amc$ (implies \ref{itm:fact});
  \item $\qCom'(\vec{d'})\subseteq\Amc$ (implies \ref{itm:foil});
  \item $q'(\vec{c'})$ is subset minimal, because $q(\vec{c})$ is (implies~\ref{itm:justification}).
 \end{enumerate}
 To see that also the second part of~\ref{itm:entailment} is satisfied, $\Tmc,q'(\vec{d'})\models C(b)$, we observe that our construction ensures that  there is a homomorphism from $q(\vec{d})$ into $q'(\vec{d'})$ which also maps $b$ again into $b$.
 The remaining item is~\ref{itm:conflict}, namely that $$\Tmc,(\Amc\setminus\Cmc)\cup q'(\vec{d'})\not\models\bot.$$

 Recall that our types were extracted from a particular model $\Imc$ of $\tup{\Tmc,(\Amc\setminus\Cmc)\cup q(\vec{d})}$.
 We construct a new interpretation $\Jmc$ based on $\Imc$ as follows:
 \begin{itemize}
  \item $\Delta^\Jmc=\{\tp_\Imc(d)\mid d\in\Delta^\Imc\}$
  \item $c^\Jmc=c^\Imc$ for all individuals $c$ occurring in $\Kmc$ or the CP,
  \item $d_{t}^\Jmc=t$ for all types $t\in\Delta^\Jmc$ (recall that we defined $d_t$ as a individual name based on $t$ above),
  \item $A^\Jmc=\{\tp_\Imc(d)\mid d\in A^\Imc\}$ for all $A\in\NC$,
  \item $r^\Jmc=\{\tup{\tp_\Imc(d),\tp(e)}\mid \tup{d,e}\in r^\Imc\}$.
 \end{itemize}
 Because all individuals from $\Amc$ are interpreted in the same way in $\Jmc$ and $\Imc$, we have $\Jmc\models\Amc$.
 Moreover, our construction preserves the types of all individuals, so that also $\Jmc\models\Tmc$.
 Finally, our construction of $\vec{d'}$ ensures that also $\Jmc\models q'(\vec{d'})$. As a result, $\Jmc$ is a model of $\tup{\Tmc,(\Amc\setminus\Cmc)\cup q(\vec{d})}$, and also~\ref{itm:conflict} is satisfied. We obtain that $E'$ is a CE, and since $\Cmc$ was
 not changed and was already minimal, $E'$ is conflict-minimal as well.
\end{proof}

\ThmConflictComplexityUpper*
\begin{proof}
  The lower bounds follow again from~\Cref{lem:abduction-reduction} and the corresponding results
for flat signature-based ABox abduction~\cite{Koopmann21a}. We thus only need to show the upper bound.
We do this by providing a decision procedure.

  Let $\tup{\Kmc,C,a,b}$ be a CP, where $\Kmc=\tup{\Tmc,\Amc}$.
  All complexity bounds are above \ExpTime, so that we can iterate over all
  the possible candidates for the conflict set $\Cmc$, which has to be a conflict set that is
  smaller than the conflict set in the CE to be verified (whether wrt. $\subseteq$ or wrt. $\leq$
  does not make a difference here).
  For simplicity, we thus focus on the problem of deciding whether there exists a syntactic CE for the given CP with a fixed conflict set $\Cmc$.
  \patrick{Did we prove the result for the complementary problem? $\Cmc$ is not minimal if there exists
  a smaller conflict set---we decide whether a CP with fixed conflict set exists!}

  We make use of the type construction used also for \Cref{the:conflict-free-size-upper}. Recall that we defined $\Sub$ to be the set of all subconcepts occuring in the CP, plus the concept. We call a subset $t\subseteq\Sub$ a \emph{valid type} if it satisfies:
  \begin{itemize}
   \item for all $\neg C\in\Sub$, $\neg C\in t$ iff $C\not\in t$,
   \item for all $C\sqcap D\in\Sub$, $C\sqcap D\in t$ iff $C,D\in t$,
   \item for all $C\sqsubseteq D\in\Tmc$, if $C\in t$, then $D\in t$,
  \end{itemize}
  Starting from the set $\mathbf{T}_0$ of all valid types, we compute the set $\mathbf{T}$ of \emph{possible types} by step-wise removing types s.t. $\exists R.C\in t$ and there is no type $t'$ in the current set s.t.
  \begin{itemize}
  \item $C\in t'$,
  \item for all $\exists R.D\in\Sub\setminus t$, $D\not \in t'$, and
  \item for all $\exists R^-.D\in\Sub\setminus t'$, $D\not\in t$.
  \end{itemize}
  This computes $\mathbf{T}$ in deterministic exponential time.

   We then non-deterministically assign to every individual name $c$ in the CP a possible type $\tp(c)$ so that the following is satisfied:
  \begin{enumerate}
   \item If $A(c)\in\Amc\setminus\Cmc$, then $A\in\tp(c)$,
   \item If $r(c,d)\in\Amc\setminus\Cmc$ and $\exists r.C\in\Sub\setminus t$, then $C\not\in\tp(d)$,
   \item If $r(c,d)\in\Amc\setminus\Cmc$, $\exists r.C\in\Sub$ and $C\in\tp(d)$, then $\exists r.C\in \tp(c)$.
  \end{enumerate}
  Note that there are exponentially many choices possible for this, since the number of individual names is polynomially bounded and the set of types exponentially.
  In addition, to every possible type $t$, we introduce a fresh individual name $c_t$ and set $\tp(c_t)=t$.
  We now collect the set $\Pmf$ of pairs $\tup{c,d}$ where $c$ is an individual from the input and $d$ is an individual from the input or one of the introduced individual names.

  From this set of pairs, we now construct a CE similarly to the way we did in the proof for \Cref{lem:dif-min}. Specifically,
  we define the vector $\vec{x}$ to contain one variable $x_{c,d}$ for every pair $\tup{c,d}\in\Pmf$. $\vec{c}$ assigns to every such variable $x_{c,d}$ in $\vec{x}$ the individual name $c$, and $\vec{d}$ assigns to it $d$. The combined pattern $q(\vec{x})$ contains an assertion $A(x_{c,d})$ if $A(c)\in\Amc$ and $A\in \tp(d)$. It contains an assertion $r(x_{c_1,d_1},x_{c_2,d_2})$ if
  \begin{enumerate}
  \item  $r(c_1,c_2)\in\Amc$,
  \item for every $\exists r.D\in\Sub$, if $D\in \tp(d_2)$ then $\exists r.D\in\tp(d_1)$, and if $\exists r.D\not\in\tp(d_1)$, then $D\not\in\tp(d_2)$, and
  \item for every $\exists r^-.D\in\Sub$, if $D\in\tp(d_1)$ then $\exists r^-.D\in\tp(d_2)$, and if $\exists r^-.D\not\in\tp(d_2)$, then $D\not\in\tp(d_1)$.
  \end{enumerate}
  The ABox pattern $q(\vec{x})$ is split into its two components $\qCom(\vec{x})$ and $\qDif(\vec{x})$
  in the obvious way so that \ref{itm:fact} and \ref{itm:foil} are satisfied. Condition \ref{itm:justification} is not
  important in the context of the procedure, since it is always possible to minimize $q(\vec{d})$ accordingly.

  We call every structure $\tup{\qCom(\vec{x}),\qDif(\vec{x}),\vec{c},\vec{d},\Cmc}$ that can be obtained in this way
  (i.e. via the non-deterministic choice of the function $\tp$) a \emph{CE-candidate}, and observe that there are exponentially many
  of such candidates, which can be deterministically iterated in exponential time.
  We can show that every CE-candidate satisfies \ref{itm:conflict}. For this, we need to construct a model for $\tup{\Tmc,(\Amc\setminus\Cmc)\cup q(\vec{d})}$. For this, we introduce a domain element for each individual name used in our construction, and follow the type assignments by $\tp$ to define the extension of concept and role names, similar as we did in the proof
  for \Cref{the:conflict-free-size-upper}.

  We next observe that if a syntactic CE with conflict set~$\Cmc$ exists, then at least one CE candidate satisfies
  \ref{itm:entailment}. To see this, we observe that the CE that is constructed in the proof for \Cref{the:conflict-free-size-upper} would be a component-wise subset of at least one CE-candidate (modulo renaming of the fresh individual names).

  To complete our decision procedure, we thus only need to add a test to verify whether one of the CE candidates satisfies
  \ref{itm:entailment}, namely that $\Tmc\cup q(\vec{c})\models C(a)$ and $\Tmc\cup q(\vec{d})\models C(b)$. In the case of \ELbot, this can be decided in exponential time (polynomial in the size of $q(\vec{d})$), which establishes the \ExpTime upper bound for this DL. In the case of \ALCI, we have to be a bit more clever. In particular, we have to determine whether there exists no model $\Imc$ of $\Tmc\cup q(\vec{d})$ s.t. $\Imc\not\models C(b)$. We describe a non-deterministic procedure to decide whether such a model exists, thus establishing our $\coNExpTime$-upper bound. The procedure guesses an alternative assignment of our possible types to $\vec{d}$ that is compatible with the assertions in $q(\vec{d})$, and verifies that the type assigned to $b$ does not contain $C$. If this is successful, we can use this assignment to construct a model $\Imc$ that witnesses $\Tmc\cup q(\vec{d})\not\models C(b)$. To see that this method is complete, assume there exists a model $\Imc$
  s.t. $\Imc\not\models C(b)$. We can then extract an assignment of possible types to the individual names in $\vec{d}$ to it. We obtain that our for verifying \Cref{itm:entailment} is sound and complete, and thus established the \coNExpTime upper bound for the case of \ALCI.
\end{proof}


\subsection{The case without fresh individuals}

\ThmNoFreshEL*

\begin{proof}
	We prove this theorem in two parts: Lemma~\ref{lem:NoFreshEL} proves the claim for $\ELbot$, and Lemma~\ref{lem:NoFreshALC} for $\ALC$ and $\ALCI$.
	The statement for $\ALC$ and $\ALCI$ also relies on \Cref{lem:ver-to-bounded}.
\end{proof}

\begin{lemma}\label{lem:NoFreshEL}
	Deciding whether a given syntactic 
	explanation without fresh individuals is (subset or cardinality) conflict-minimal is $\co\NP$-complete for $\ELbot$.
\end{lemma}
\begin{proof}
	For membership, given a CE $E$, one can guess an explanation $E'$ as a counter example with smaller conflict set.
	The verification that $E'$ is a valid CE can be performed in polynomial time.
	
	For hardness, we reduce from the complement of propositional satisfiability.
	Let $\varphi = \{c_1,\dots,c_n\}$ be a CNF formula over variables $X =\{x_1,\dots,x_n\}$, where each $c_i$ is a clause.
	A literals $\ell$ is a variable $x$ or its negation $\neg x$.
	For a literal $\ell$, we denote by $\bar\ell$ its ``opposite'' literal.
	For set $X$ of variables, we let $\Lit(X)$ denote the collection of literals over $X$.
	Let $\calK = \tup{\calT,\calA}$ be a KB with concept names
	$\{A_x, A_{\bar x} \mid x\in X\} \cup \{A_c \mid c \in\varphi\} \cup \{C, P_t, Q_t, N_t, V \}$,
	where
	\begin{align*}
		\calT = &\; \{A_x \sqcap A_{\bar x} \sqcap V \subsum \bot \mid x\in X\} \\
		& \cup \{A_\ell \subsum A_c \mid \ell \in c, c\in\varphi\} \\
		& \cup \{\bigsqcap_{c\in \varphi}\nolimits A_c \subsum C \} \\
		& \cup \{P_t\sqcap Q_t\subsum C,Q_t\sqcap N_t\subsum \bot \}, \\
		\calA = & \;\{A_\ell(a) \mid \ell\in\Lit(X)\} \\
		& \cup \{V(b), P_t(a), Q_t(a), N_t(b)\}.
	\end{align*}
	Now, we let $P =\tup{\calK, C, a, b}$.
	We have $\calK\models C(a)$ and $\calK\not\models C(b)$ (the only axioms in $\calK$ involving $b$, namely $V(b)$ and $N_t(b)$
	cannot entail $C(b)$).
	Finally, let $E=\tup{\emptyset, \{P_t(x), Q_t(x)\}, \tup{a}, \tup{b}, \{N_t(b)\}}$.
	Then, $E$ is a valid CE since 
	(1) $\calT, \{P_t(a), Q_t(a)\}  \models C(a)$, $\calT, \{P_t(b), Q_t(b)\}  \models C(b)$
	(2) $\{P_t(a), Q_t(a)\}\subseteq \calA$ , and
	(3) $\{N_t(b)\}$ is subset-minimal, such that $\calT, \calA\setminus \{N_t(b)\},  \{P_t(b), Q_t(b)\} \not \models \bot$.
	The correctness of our reduction is established by proving the following claim.
\begin{claim}
	$\varphi$ is satisfiable iff $E$ is not a (subset) conflict-minimal CE for $P$.
\end{claim}
\begin{claimproof}
	($\Rightarrow$)
	Suppose $\varphi$ is satisfiable and let $\theta\subseteq \Lit(X)$ be a satisfying assignment, seen as a collection of literals.
	We let $\qDif\dfn \{A_\ell(x) \mid \ell \in \theta\}$ be an ABox pattern.
	Now, we set $E'=\tup{\emptyset, \qDif(x), \tup{a}, \tup{b}, \emptyset}$.
	It is easy to observe that $E'$ is a valid CE for $P$.
	Indeed, there is  $\ell\in \Lit(X)$ for each clause $c\in\varphi$ s.t. $\ell\in \theta\cap c$.
	As a result, $\calT, \qDif(a)\models A_c(a)$ for each $c\in\varphi$ and hence $\calT, \qDif(a)\models C(a)$.
	Likewise, we have $\calT, \qDif(b)\models C(b)$.
	Finally, $\qDif(a)\subseteq \calA$ since $\theta\subseteq \Lit(X)$. 
	Observe that no conflicts are triggered due to the axiom $A_x\sqcap A_{\bar x}\sqcap V\subsum \bot$, since $\theta$ is a valid assignment and hence either $x\in\theta$ or $\bar x\in\theta$ for each $x\in X$.
	That is, either $A_x(b)\in \qDif(b)$ or $A_{\bar x}(b)\in \qDif(b)$ but not both, for any $x\in X$.
	We conclude the proof in this direction by observing that $E$ is not conflict minimal CE for $P$ since $E'$ is a valid CE with empty conflict set.

	($\Leftarrow$)
	Suppose $E$ is not a (subset) conflict-minimal CE for $P$. We prove that $\varphi$ is satisfiable.
	Suppose to the contrary that $\varphi$ is not satisfiable.
	That is, for each assignment $\theta\subseteq \Lit(X)$, there is a clause $c\in \varphi$ such that $\theta\cap c=\emptyset$.
	Since $E$ contains a single assertion in its conflict set, let us observe whether we can obtain an explanation without any conflicts.
	Observe that for any $q(\vec c)$ the only way to achieve the entailment
	$\calT, q(\vec c)\models C(a)$ without triggering conflicts is via literals $A_\ell(a)$ since $q(\vec c)\subseteq \calK$ must be true.
	Moreover, $q(\vec c)$ must consider at most one assertion from the set $\{A_x(a), A_{\bar x}(a)\}$ since otherwise $\{A_x(b), A_{\bar x}(b)\}\subseteq q(\vec d)$ would imply a conflict $\{V(b)\}\subseteq \calA$ due to $A_x\sqcap A_{\bar x}\sqcap V\subsum \bot$.
	But this implies that there is no explanation with empty conflict set, since $\varphi$ is not satisfiable.
	In other words, it is not possible to use a valid assignment (and hence without triggering the conflict $\{V(b)\}$) and entail $C(a)$ since for each such valid assignment $\theta$, there exists at least one clause $c\in\varphi$ that is not satisfied by $\theta$ (and hence $A_c(a)$ is not entailed by its corresponding ABox).
	This completes the proof of our claim
\end{claimproof}
	Observe that the claim also applies to the case of cardinality-minimal CEs since $E$ only contains a conflict of size one.
	We conclude by observing that the reduction can be obtained in polynomial time.
\end{proof}

\begin{lemma}\label{lem:NoFreshALC}
	Deciding the existence of a syntactic CE without fresh individuals and with conflict
	of size at most $n$ is $\ExpTime$-complete for $\ALC$.
\end{lemma}

\begin{proof}
	The membership follows, since one can try all possible explanations (that trigger an inconsistency) such that the corresponding conflict set $\calC$ has size at most $n$ in exponential time.

	For hardness, we consider the following reduction from instance checking for $\ALC$.
	Let $\calK= \tup{\calT,\calA}$ be a KB and $A(a)$ be an instance query.
	We let $\calK'=\tup{\calT',\calA'}$ be another KB where $\calT'= \calT\cup \{A\sqcap C\subsum \bot\}$ and $\calA'= \calA\cup \{C(b)\}$ for a fresh concept name $C$ and individual $b$.
	Then, we let $P= \tup{\calK', C,b,a}$ be our CP.
	Clearly, $\calK'\models C(b)$ (trivially) and $\calK'\not\models C(a)$ since $C(a)\not\in \calA'$ and no axiom in $\calA$ involves the concept names $C$ or individual $b$.
	Now, we prove the following claim.
	\begin{claim}
		$\calK\not \models A(a)$ iff $P$ admits a syntactic explanations in $\calK'$ with empty conflict set.
	\end{claim}
	\begin{claimproof}
		Observe that the only syntactic CE for $P$ is of the form $E\dfn \tup{\emptyset, \{C(x)\}, \tup{b}, \tup{a}, \calC}$ for some conflict set $\calC$.
		This holds since $C$ does not appear in $\calK$. Hence neither $C(b)$  nor $C(a)$ can be entailed from $\calK$.
		Then, $\calK\models A(a)$ iff $\calK'\models A(a)$.
		Consequently, $E$ triggers an inconsistency due to the axiom $A\sqcap C\subsum \bot$.
		Conversely, the only possibility when an inconsistency arises by $E$ is when $\calK'\models A(a)$.
		Therefore, we have that $\calK\not\models A(a)$ iff $E$ is a CE for $P$ with $\calC=\emptyset$ iff $P$ admits a CE with empty conflict set (i.e., $|\calC|=0$).
	\end{claimproof}
	This completes the proof to our theorem.
\end{proof}

\section{Commonality-Maximal Explanations}

Regarding commonality-maximal explanations, we only have results in the case of cardinality-maximality.

\ThmELGroundSynSimSize*

\begin{proof}
	Once again, we prove this theorem in two parts: Lemma~\ref{lem:commonalityEL} for $\EL/\ELbot$ and Lemma~\ref{lem:commonalityALC} for $\ALC$.
	As before, we rely on \Cref{lem:ver-to-bounded}.
\end{proof}

\begin{lemma}\label{lem:commonalityEL}
	For $\EL/\ELbot$, deciding the existence of a syntactic CE with commonality of size at least $n$ is $\NP$-complete.
\end{lemma}

\begin{proof}
	The membership is easy since one can guess ABox patterns $\qCom(\vec x)$, $\qDif(\vec x)$ with $|\qCom|= n$, $|\qCom\cup \qDif (\vec c)|\leq |\calA|$ and vectors $\vec c$, $\vec d$, such that $\tup{\qCom(\vec x), \qDif(\vec x), \vec c, \vec d, \emptyset}$ is a valid CE.
	Since $q(\vec c)\subseteq \calA$, we have an upper bound on the size of guessed explanation.
	The verification requires polynomial time. This gives membership in $\NP$.
	For $\ELbot$, we have NP-membership following the same argument as in the case of difference-minimal CEs (proof of \Cref{thm:el-ground-syn-size}).

	For hardness, we reduce from the clique problem.
	That is, given a graph $G=(V,E)$ and a number $n\in\mathbb N$, is there a set $S\subseteq V$ s.t. $(u,v)\in E$ for each pair $\{u,v\}\subseteq V$.
	For the reduction, we let $Y_i,X_i$ be concept names for $i\leq n$ and $e,s$ be two distinct role names.
	Then, consider the KB $\calK=\tup{\calT,\calA}$ with
	\begin{align*}
		\calT =  \{\,& \exists e.X_1 \sqcap \dots \exists e.X_{n-1} \subsum Y_n, \\
		& \qquad \vdots \\
		& \exists e.X_1 \sqcap \exists e.X_3\dots \sqcap \exists e.X_{n}  \subsum Y_2, \\
		& \exists e.X_2 \sqcap \dots \sqcap \exists e.X_{n} \subsum Y_1, \\
		& \exists s.Y_1 \sqcap \dots \sqcap \exists s.Y_{n} \subsum Q,  \\
		& Q \sqcap D \subsum  C\} \\
		\calA = &  \{\, e(u,v), e(v,u) \mid (u,v)\in E\, \} \cup \\
		& \{\,s(a,v),s(b,v) \mid v\in V  \, \} \cup\\
		& \{\,X_i(v) \mid i\leq n \, \} \cup \{\, D(a)\,\}
	\end{align*}
	Moreover, we let $P=\tup{\calK, C,a,b}$ be our CP.
	We prove the correctness of our reduction via the following claim.
	\begin{claim}
		$G$ has a clique of size $n$ iff $P$ has a CE with $|\qCom(\vec c)|=n'$ where $n' = 2n + n\times(n-1)$.
	\end{claim}

	\textbf{Intuitively} we enforce with the TBox-axioms that individuals in $X_i$ (and entailed to be in $Y_i$) constitute a clique. Thereby, we have the entailment of $Y_i(v)$ if and only if, it has edges to every individual $u$ in $X_j$ for $i\neq j$.
	This simulates the effect that in each CE for $P$, $\qCom(\vec c)$ have (1) $n$ distinct assertions of the form $s(a,v_i)$ with $i\leq n$,
	(2) $n$ distinct assertions of the form $X_i(v_i)$ with $i\leq n$, and
	(3) $n(n-1)$ distinct assertions of the form $e(v_i,v_j)$ for $i\neq j\leq n$.
	Having some non-distinct assertions in this collection will cause repeating certain assertions for the entailment of some $Y_i(v_i)$.
	As a result, one obtains a \emph{smaller} CE which has $|\qCom|< n'$ (a contradiction).
	That is, one can not artificially increase the size of $\qCom$ without violating the subset-minimality of a CE.

	\begin{claimproof}
		``$\Longrightarrow$''.
		Let $S=\{a_1,\dots, v_n\}$ be a clique of size $n$ in $G$.
		Further, let $\vec x = \tup{z,x_1,\dots,x_n}$ and consider the following ABox patterns.
		\begin{itemize}
			\item $\qCom(\vec x) = \{X_i(x_i), s(z,x_i) \mid i\leq n\} \cup \{e(x_i,x_j) \mid v_i,v_j\in S \}$
			\item $\qDif(\vec x) = \{D(z)\}$
			\item $\vec c = \tup{a,v_1,\dots,v_n}$
			\item $\vec d = \tup{b, v_1,\dots,v_n}$
		\end{itemize}
		Then, we prove that $E_s = \tup{\qCom(\vec x), \qDif(\vec x), \vec c, \vec d,\emptyset}$ is a valid CE for $P$ and that $|\qCom(\vec c)|=n'$.
		Notice that $(v_i,v_j)\in E$ for each $\{v_i,v_j\}\subseteq S$.
		Consider the collection $q(\vec c)\subseteq \calA$ of ABox assertions.
		Then, $\{X_i(v_i),s(a,v_i)\}\subseteq  q(\vec c)$ for each $v_i\in S$ and $i\leq n$.
		Moreover, $e(v_i,v_j)\in q(\vec c)$ for each $v_i,v_j\in S$ with $i\neq j$ since $S$ is a clique.

		\textbf{(I)} { We prove that $\calT,q(\vec c)\models C(a)$.}
		Let $i\leq n$ be fixed.
		Then, $(v_i,v_j)\in E$ for each $v_j\in S$ such that $i\neq j \leq k$.
		As a result, $\calT, \{e(v_i,v_j),X_j(v_j)\}\models \exists e.X_j(v_i)$ for every $i\neq  j\leq n$.
		Additionally, this implies that $\calT, \bigcup_{i\neq j\leq n}\{e(v_i,v_j),X_j(v_j)\}\models Y_i(v_i)$ for each $i\leq n$.
		Moreover, since $\calT, \{s(a,v_i),Y_i(v_i)\}\models \exists s.Y_i(a)$ for each $ i\leq n$, we have that $\calT, q(\vec c)\models C(a)$.
		Which together with $D(a)\in \qCom(\vec c)$ results in $\calT,q(\vec c)\models C(a)$.
		The claim that $\calT,q(\vec d)\models C(b)$ can be proven analogously by simply replacing $a$ and $\vec c$ by $b$ and $\vec d$, respectively.

		\textbf{(II)} $q(\vec c)\subseteq \calA$, $\qCom(\vec d)\subseteq \calA$ follows immediately.

		\textbf{(III)} $q(\vec c)$ is subset-minimal.
		Notice that, since $S$ is a clique of size $n$, all assertions in $q(\vec c)$ are distinct and necessary to get the entailment $\calT, q(\vec c)\models C(a)$.
		Suppose to the contrary, let $q'(\vec c)\subset q(\vec c)$.
		Then, we have the following cases. (A) $D(a)\in q'$ must be true since no other axiom in $\calT$ includes the concept $D$.
		(B) Suppose $X_i(v_i)\not\in q'$ for some $i\leq n$. Then, for $j\neq i$, the entailment $Y_j(v_j)$ can not be established by $q'$.
		This holds since $\{e(v_j,v_i), X_i(v_i)\}\models \exists e.X_i(v_j)$ and without $X_i(v_i)$, one can not obtain $\exists e.X_i(v_j)$ which is necessary for $Y_j(v_j)$ for each $j\neq i$ (recall that $q'(\vec c)\subseteq q(\vec c)$ and therefore $q'(\vec c)$ can not contain a different $X_i(v')\not\in q(\vec c)$).
		Thus $\calT,q'(\vec c)\not\models C(a)$.
		(C) The case for $s(a,v_i)\not\in q'(\vec c)$ and $e(v_i,v_j)\not\in q'(\vec c)$ follows via similar arguments.
		As a result, no proper subset of $q(\vec c)$ yields an explanation for $P$.

		Hence, it follows from (I)--(III) that $E_s$ is a valid CE for $P$.
		We conclude by observing that $|\qCom(\vec c)|=2n + n\times (n-1)$ due to $n$ distinct assertions of the form $X_i(x_i)$ and $s(a,x_i)$ each, and $n\times(n-1)$ assertions due to $e(v_i,v_j)$ for $i\neq j\leq n$.

		``$\Longleftarrow$''.
		Let $E=\tup{\qCom(\vec x), \qDif(\vec x), \vec c, \vec d,\emptyset}$ be a CE for $P$ with $|\qCom(\vec c)|= n'$.
		Since $q(\vec c)$ is subset-minimal, removing any assertion from it would break the entailment $\calT,q(\vec c)\models C(a)$.

		We have the following observations regarding the assertions in $q(\vec c)$.
		First, to achieve $\calT,q(\vec c)\models C(a)$ requires to have $\calT,q(\vec c)\models Q(a)$ since $q(\vec c)\subseteq \calA$.
		Furthermore, to achieve $\calT,q(\vec c)\models Q(a)$, one requires $\calT,q(\vec c)\models \exists s.Y_i(a)$ for each $i\leq n$.
		Since $s(a,v)\in\calA$ for each $v\in V$, we can take $s(a,v_i)\in \qCom(\vec c)$ for each $i\leq n$.
		However, $Y_i(v_i)\not\in \calA$, therefore we must have $\calT,q(\vec c)\models Y_i(v_i)$.
		Now, this can only be achieved if for each such $i\leq n$, $\calT,q(\vec c)\models \bigsqcap_{i\neq j\leq n} \exists e.X_j(v_i)$.
		This in turn requires $e(v_i,v_j),X_j(v_j)$ to be in $q(\vec c)$.
		This holds since $X_j(v)\in\calA$ for each $v\in V$ and $j\leq n$.
		We will see that $e(u,v)\in \qCom(\vec c)$ and $X_j(u),X_j(v)\in \qCom(\vec c)$ for $u,v\in V$ using an argument on the size of $\qCom(\vec c)$.

		Using the above observations, we next establish that $\qCom(\vec c)$ contains assertions of the following form (a) $s(a,v_i)$, (b) $X_i(v_i)$, (c) $ e(v_i,v_j), e(v_j,v_i)$ for $i\neq j\leq n$ and $n$ distinct nodes $v_1,\dots,v_n\in V$.
		In other words, $\qCom(\vec c)$ takes the form of the following \emph{matrices}.
		This is achieved via observations (O1--O5) in the following.
		\[
		\resizebox{0.5\linewidth}{!}{$
			\begin{bmatrix}
				- & e(v_1,v_2) & \dots & e(v_1,v_n) \\
				e(v_2,v_1) & - & \dots & e(v_2,v_n) \\
				\vdots &  & \vdots & \vdots \\
				e(v_n,v_1) &  & \dots & - \\
			\end{bmatrix}
			\begin{bmatrix}
				X_1(v_1) \\
				X_2(v_2) \\
				\vdots    \\
				X_n(v_n)
			\end{bmatrix}
			\begin{bmatrix}
				s(a,v_1) \\
				s(a,v_2) \\
				\vdots    \\
				s(a,v_n)
			\end{bmatrix}
			$}
		\]
		%
		\textbf{(O1)} $D(a)\in \qDif(\vec c)$, since $D(b)\not\in \calA$ and no other axiom in $\calT$ involves $D$. Thus $D(a)\not\in \qCom(\vec c)$.\\
		\textbf{(O2)} For each $X_i$, there is exactly one $v\in V$, s.t. $X_i(v)\in \qCom(\vec c)$.
		Moreover, such an element $v\in V$ is unique for each $i\leq n$.

		We prove this via the following three sub-claims.

		\textbf{(Claim-I)} For each $i\leq n$, there is at least one $v\in V$ s.t. $X_i(v) \in \qCom(\vec c)$. Suppose to the contrary that for some $i\leq n$, $\qCom(\vec c)$ does not contain any assertion of the form $X_i(v)$.
		Then we observe that the entailment of $Y_j(u)$ can not be achieved for any $i \neq j \leq n$ and $u\in V$.
		This holds since $\calT, q(\vec c)\models Y_j(u)$ is true only if $u$ has an $e$-neighbor in $X_\ell$ for each $j \neq\ell \leq n$ including $\ell=i$ (due to the axioms $\bigsqcap_{j \neq \ell \leq n} \exists e.X_\ell\subsum Y_j$). This contradicts our assumption as $X_i(v)\in\calA$ and hence $X_i(v)\not\in \qDif(\vec c)$.

		\textbf{(Claim-II)} For each $i\leq n$, there is at most one $v\in V$ s.t. $X_i(v) \in \qCom(\vec c)$.
		Suppose to the contrary that for some $i\leq n$, there are distinct $v\neq v'$ s.t. $X_i(v), X_i(v')\in \qCom(\vec c)$.
		Now, consider the entailment of $Y_i(v)$.
		It must be the case that $\calT,q(\vec c)\models \bigsqcap_{i \neq j \leq n} \exists e.X_j(v)$ which requires $\{e(v,v_j),X_j(v_j)\mid i \neq j\leq n\}\subseteq q(\vec c)$, thus each such $v$ must have edges to all $v_j$s s.t. $X_j(v_j)\in q(\vec c)$.
		As a result, it suffices to keep only one of $\{X_i(v),X_i(v')\}$ in $q(\vec c)$. But this leads to a contradiction since $q(\vec c)$ has to be subset-minimal.

		\textbf{(Claim-III)} For each $i\leq n$, there is a unique $v\in V$ s.t. $X_i(v) \in \qCom(\vec c)$.
		That is, for any $i\neq j\leq n$, there is no $v\in V$ s.t. $X_i(v), X_j(v)\in \qCom(\vec c)$.
		Suppose to the contrary, and look at the entailment of $Y_j(v)$.
		The entailment requires $\calT,q(\vec c) \models \exists e.X_i(v)$, which in turn is the case if $\{e(v,v'),X_i(v')\}\subseteq q(\vec c)$.
		But this implies that $v'=v$ due to (Case-II), which leads to a contradiction since $e(v,v)\not\in \qCom(\vec c)$ for any $v\in V$.
		Furthermore, allowing $e(v,v)\in \qDif(\vec c)$ would reduce the number of assertions in $\qCom(\vec c)$ due to the minimality (since it only contains necessary assertions).

		Additionally, the following observations can be proved easily.\\
		\textbf{(O3)} For each $i\neq j\leq n$: $X_i(v_i),X_j(v_j)\in q(\vec c) \iff e(v_i,v_j),e(v_j,v_i)\in q(\vec c)$. \\
		\textbf{(O4)} For each $i\leq n$: $X_i(v_i)\in q(\vec c) \iff s(a,v_i)\in q(\vec c)$.\\
		\textbf{(O5)} For each $i\leq n$: $X_i(v_i)\in q(\vec c) \iff \calT,q(\vec c)\models Y_i(v_i)$.\\
		%

		Now, we let $S=\{v \mid X_i(v)\in q(\vec c) \text{ for } i \leq n\}$.
		Then, $|S|=n$, since there are exactly $n$ distinct assertions of the form $X_i(v)$ in $q(\vec c)$.
		The claim that $S$ is a clique follows due to (O3).
		Suppose there are $v_i,v_j\in S$ such that $(v_i,v_j)\not\in E$.
		Now, look at the entailment $\calT,q(\vec c)\models Y_i(v_i)$.
		Since $X_j(v_j)\in \qCom(\vec c)$, one needs additionally $e(v_i,v_j)\in \qCom(\vec c)$ to get the entailment $\calT, \{e(v_i,v_j), X_j(v_j)\} \models \exists e.X_j(v_i)$ for each $i\neq j\leq n$.
		However, $e(v_i,v_j)\not\in \qCom(\vec c)$ since $e(v_i,v_j)\not\in \calA$ due to $(v_i,v_j)\not\in E$.
		This leads to a contradiction. Hence $(v_i,v_j)\in E$ for each $v_i,v_j \in E$.
		Observe that although $e(v_i,v_j)\in \qDif(\vec c)$ achieves the required entailment, it results in a subset of $\qCom(\vec c)$ in an explanation, which violates the claim regarding the size of $\qCom(\vec c)$.
	\end{claimproof}
	This completes the proof to our theorem by observing that the reduction can be achieved in polynomial time.
	%
\end{proof}

\begin{lemma}\label{lem:commonalityALC}
	For $\ALC$, deciding the existence of a syntactic CE with commonality of size at least $n$ is $\EXP$-complete.
\end{lemma}

\begin{proof}
	The membership follows, since one can try all possible explanations with $\qCom$ of size at least $n$ in exponential time.

	For hardness, we reconsider the reduction from instance checking for $\ALC$ established in the proof of \Cref{thm:alc-ground-syn-size}.
	Given a KB $\calK= \tup{\calT,\calA}$ and an instance query $A(a)$, we let $\calK'=\tup{\calT',\calA'}$ be another KB where 
	\begin{align*}
		\calT' &= \calT\cup \{A\sqcap B\subsum C\},\\
		\calA' &= \calA\cup \{A(b),B(b)\} \\
		& \quad \cup \{D(b),r(b,c),r(d,b)\mid D(a),r(a,c), r(d,a)\in \calA\}
	\end{align*}
	for fresh concept names $B,C$ and individual $b$.
	Moreover, we let $P= \tup{\calK', C,b,a}$ be our CP.
	Then, it holds that $\calK\models A(a)$ iff $P$ admits a syntactic explanation in $\calK'$ with $\qDif$ of size at most one (Claim~3) iff $P$ admits a syntactic ground explanation in $\calK'$ with $\qCom$ of size at least one.
\end{proof}

\clearpage

\section{Details on the Implementation and the Experiments}

\subsection{The Implemented Method}

We give a more detailed description on how we constructed the sets $\Amc'$ and $\Ibf$ used by the refined super
structure $E_m$:
\begin{description}
 \item[ABox] $\Amc'$ is the union of justifications of~$C(a)$, which we compute differently depending on the
 DL in question. For \ALCI, we simply compute all ABox justifications for $C(a)$, i.e. all
 justifications with the TBox as fixed component.\footnote{We also tried the system from~\cite{DBLP:conf/esws/ChenMPY22} for computing unions of justifications, but found it to be slower for our inputs.}
 For \ELbot, we implemented a more efficient procedure using
 \Evee and \ELK. \ELK is an efficient rule-based reasoner that supports \ELbot. \Evee is a library
 for extracting proofs from reasoners such as \ELK. In particular, \Evee can generate a structure
 that contains all possible inferences that \ELK would perform when deriving a given axiom. From this structure,
 we can direcly read off the union of justifications.

 \item[Individuals] By \Cref{lem:restricting-x} and~\ref{claim:conflict}, we can always find a contrastive explanation if $\Ibf$
 contains at least as many individuals as occur in $\Amc'$.
 To avoid explanations that use arbitrary individuals, we focus on individuals that are related to the foil.
 Define the \emph{role-distance} between two individuals $a$ and $b$ in an ABox $\Amc$ as the minimal number of role assertions
 in $\Amc$ that connect $a$ and $b$. We include in $\Ibf$ 1) all individuals that have in $\Amc$ a role-distance to the foil
 individual that is at most equal to the role distance of an individual in $\Amc'$ to the fact individual,
 and 2) if needed, a number of fresh individual names to ensure $\Ibf$ has at least as many individuals as $\Amc'$.
\end{description}

For computing justifications with fixed components, we modifed the implementation of the justification algorithm
that we downloaded from Github.\footnote{\url{https://github.com/matthewhorridge/owlexplanation}, unchanged since six years.} This implementation constructs a justification by varying a set of axioms until
removing any axiom would break the desired entailment. Our modifications made sure in all places in the code that elements from the fixed set cannot be removed from this variable set of axioms, and then remove
the fixed set from the resulting set of axioms.

The algorithm for computing a difference-minimal CE based on $E_m$ is now as follows, where each step is a refined version of \ref{p:make-consistent}--\ref{p:minimize-conflict}:
\begin{enumerate}[label=\textbf{P\arabic*'}]
 \item While $\Tmc,  (\qCom\cup \qDif)(\vec{d}_m)\models\bot$, compute a justification $\Jmc$ for $\bot$ with fixed component $\Tmc\cup \qCom(\vec{d}_m)$. Remove from $\qDif(\vec{d}_m)$ some axiom from $\Jmc$ that does not invalidate the preconditions of \Cref{lem:restricting-x} and~\ref{claim:conflict}, based on an initially fixed bijection between $\vec{c}_m$ and $\vec{d}_m$.
 \item Compute a justification $\Jmc(\vec{x})$ for $\Tmc, (\qCom\cup\qDif)(\vec{x})\models C(x)$ with fixed component $\Tmc\cup\qCom(\vec{x})$. Set $\qDif(\vec{x}):= \Jmc(\vec{x})$.
 \item Compute a justification $\Jmc(\vec{x})$ for $\Tmc, (\qCom\cup\qDif)(\vec{x})\models C(x)$ with fixed component $\Tmc\cup\qDif(\vec{x})$, and set $\qCom(\vec{x}):=\Jmc(\vec{x})$.
 \item\label{p:minimize-conflict-optimized} Initialize $\Cmc:=\emptyset$, and until $\Tmc, (\Amc\setminus\Cmc)\cup (\qCom\cup\qDif)(\vec{d}_m)\models\bot$, compute a justification $\Jmc$ for it with fixed component $\Tmc\cup (\qCom\cup\qDif)(\vec{d}_m)$,
 and add some axiom of $\Jmc$ to $\Cmc$.
\end{enumerate}
The final patterns $\qCom(\vec{x})$, $\qDif(\vec{x})$ and $\Cmc$ then constitute the computed CE.

\subsection{Generating the CPs}

\newcommand{\ranCon}{\texttt{ranCon}}
\newcommand{\ranAtom}{\texttt{ranAtom}}
\newcommand{\size}{\texttt{size}}

We define the \emph{size} of an \EL concept $C$ as its tree-size, e.g.
$\size(\top)=\size(A)=1$ for $A\in\NC$, $\size(\exists r.C)=1+\size(C)$,
$\size(C\sqcap D)=\size(C)+\size(D)$.

To generate random concepts for a given KB $\Kmc$, we used a recursive procedure $\ranCon(c,n)$ that takes as parameter an individual $c$ and a size bound $n\geq 1$ and returns an \EL concept $C$ of size at most $n$ s.t. $\Kmc\models C(c)$. The procedure $\ranCon$ uses itself the procedure $\ranAtom(c,n)$ that returns a random \emph{atom}, i.e. an \EL concept that is not a conjunction, of size at most $n$.

$\ranAtom(c,n)$ proceeds as follows:
\begin{enumerate}
 \item If $n=1$ or there is no role assertion $r(c,d)\in\Kmc$, return $\top$ or a random concept name $A$ s.t. $\Kmc\models A(c)$
 \item Otherwise, pick a random Boolean value. If it is true, proceed with Step~1. Otherwise, pick a random
 role assertion $r(c,d)\in\Kmc$, set $C:=\ranCon(d,n-1)$, and return $\exists r.C$.
\end{enumerate}
For all the random choices, we used a uniform distribution. For instance, if $c$ satisfies $m$ concept names,
then each concept name, as well as $\top$, is chosen with a probability of $\frac{1}{m}$.

$\ranCon(c,n)$ now repeats the following steps until $m=0$, initializing $C:=\ranAtom(c,n)$ and $n=m-\size(C)$.
\begin{enumerate}
 \item $D=\ranAtom(c,m)$,
 \item $C:=C\sqcap D$
 \item $m:=m-\size(D)$,
\end{enumerate}

We now attempted to generate 10 random CPs $\tup{\Kmc,C,a,b}$ for a given KB $\Kmc$ as follows.
We generated the concept $C=\ranCon(a,5)$, and then tried to select a random individual $b$ s.t.
$\Kmc\not\models C(b)$ and for some $A\in\NC$, $\Kmc\models A(a)$, $A(b)$, or for some $r\in\NR$ and $a',b'\in\NI$,
$r(a,a')$, $r(b,b')\in\Kmc$.

\subsection{Comments on the Reproducibility Checklist}

\textbf{Hyper-Parameters}: Hyper-parameters did not play a central role in our evaluation, since we were not comparing different algorithms or configurations, but rather show-cased a first (unoptimized) prototype showing the general practicality of our new notion of constrastive explanations.

\textbf{Randomness and Seed Functions} We used a random number generator when creating the CPs---to keep the reproducibility simple, we fixed the seed function to 0. However, there are other random factors in the implementation that we were not able to control, and that come through the usage of hash sets in various places in the implementation. We used the standard implementation of HashSet from Java, which shows nondeterministic behavior when selecting an arbitrary element from the set. This affects for instance the computation of justifications, and the computations of repairs in Step~\ref{p:minimize-conflict-optimized}, where we would often pick the first element returned by an iterator on the HashSet. We are aware that this limits the reproducibility of our results, but decided in this way since a more deterministic implementation (e.g. using sortest sets) would significantly affect the runtimes.

\textbf{Code for Preprocessing, Conducting and Analyzing the Experiments.}
For detailed instructions on how to download the datasets, compile the code, preprocess the ontologies, run the experiments and analyze the data, look at the README.md in the root of the source directory. All the scripts needed to reproduce the experiments are in the \emph{Experiments} subfolder. The implementation
of the contrastive explanation generator is as usual in the \emph{src} folder.
We used code comments to make explicit references to the paper, in particular in the classes
\emph{ContrastiveExplanationProblem}, \emph{ContrastiveExplanation} and \emph{ContrastiveExplanationGenerator}. Here, we
give explicit links to components in the definitions and steps in our algorithm.

\end{document}